\documentclass[10pt, twocolumn]{article}

\usepackage{arxiv}
\usepackage{geometry}
\usepackage{changepage}

\usepackage{hyperref}

\usepackage{authblk}

\usepackage{algorithm}

\usepackage[utf8]{inputenc}

\usepackage{hyperref}
\hypersetup{
    colorlinks=true,
    citecolor=blue,
    filecolor=blue,
    linkcolor=blue,
    urlcolor=blue,
}
\usepackage{natbib}
\usepackage{enumitem}
\usepackage{microtype}
\usepackage{amsmath,amsthm,amssymb}
\DeclareMathOperator*{\argmin}{arg\,min}
\usepackage{dsfont}
\usepackage{mathrsfs}
\usepackage[noabbrev]{cleveref}
\usepackage{stmaryrd}

\usepackage{graphicx}
\usepackage{subcaption}
\usepackage{booktabs}
\usepackage{dblfloatfix} 
\usepackage{multirow}

\usepackage{comment}

\theoremstyle{plain}
\newtheorem{theorem}{Theorem}[section]
\newtheorem{proposition}[theorem]{Proposition}
\newtheorem{lemma}[theorem]{Lemma}
\newtheorem{corollary}[theorem]{Corollary}
\theoremstyle{definition}
\newtheorem{definition}[theorem]{Definition}
\newtheorem{assumption}{Assumption}

\theoremstyle{remark}
\newtheorem{remark}[theorem]{Remark}

\crefname{assumption}{assumption}{assumptions}

\newtheoremstyle{TheoremNum}
    {\topsep}{\topsep}              
    {\itshape}                      
    {}                              
    {\bfseries}                     
    {.}                             
    { }                             
    {\thmname{#1}\thmnote{ \bfseries #3}}
\theoremstyle{TheoremNum}

\newtheorem{customProposition}{Proposition}

\theoremstyle{definition}

\theoremstyle{definition}

\theoremstyle{definition}
\newtheorem{model}[theorem]{Model}

\usepackage{algcompatible}

\usepackage{breqn}

\usepackage[dvipsnames]{xcolor}
\definecolor{blindorange}{RGB}{215,131,37}
\definecolor{blindblue} {RGB}{81,172,226}
\definecolor{blindpurple} {RGB}{193,114,177}
\definecolor{blindgreen} {RGB}{33,145,106}
\definecolor{blindpink} {RGB}{252,174,226}
\definecolor{blindbrown} {RGB}{203,145,100}
\definecolor{tabred}{RGB}{197,58,50}

\newcommand{\ind}{\perp\!\!\!\!\perp} 

\usepackage[hang,flushmargin,bottom]{footmisc}

\usepackage{xspace}

\usepackage{wasysym}

\def\mis{\mathrm{mis}}
\def\obs{\rm{obs}}

\def\Tr{\rm{Tr}}
\def\Cal{\rm{Cal}}
\def\Test{\rm{Test}}
\def\test{\rm{test}}
\def\ind{\perp\!\!\!\!\perp }

\DeclareMathAlphabet{\mathdutchcal}{U}{dutchcal}{m}{n}

\newcommand{\camerareadyrevision}[1]{{#1}}
\newcommand{\camerareadyrevisionlast}[1]{{#1}}

\newcommand{\mask}{CP-MDA-Nested\xspace}
\newcommand{\masksub}{CP-MDA-Exact\xspace}

\definecolor{mydarkblue}{rgb}{0,0.08,0.45}
\hypersetup{
    colorlinks=true,
    linkcolor=mydarkblue,
    citecolor=mydarkblue,
    filecolor=mydarkblue,
    urlcolor=mydarkblue,
}

\usepackage{times}

\title{Conformal Prediction with Missing Values}
\author[,1,2,3]{Margaux Zaffran\thanks{Corresponding author: \texttt{margaux.zaffran@inria.fr}}}
\author[3]{Aymeric Dieuleveut}
\author[2]{Julie Josse}
\author[4]{Yaniv Romano}

\date{}

\affil[1]{Electricité De France R\&D, Palaiseau, France}
\affil[2]{PreMeDICaL project team, INRIA Sophia-Antipolis, Montpellier, France}
\affil[3]{CMAP, CNRS, École polytechnique, Institut Polytechnique de Paris, Palaiseau, France}
\affil[4]{Departments of Electrical Engineering and of Computer Science, Technion - Israel Institute of Technology, Haifa, Israel}

\begin{document}

\maketitle

\begin{adjustwidth}{0.7cm}{0.7cm}

\begin{abstract}

Conformal prediction is a theoretically grounded framework for constructing predictive intervals. We study conformal prediction with missing values in the covariates -- a setting that brings new challenges to uncertainty quantification. We first show that the marginal coverage guarantee of conformal prediction holds on imputed data for any missingness distribution and almost all imputation functions. However, we emphasize that the average coverage varies depending on the pattern of missing values: conformal methods tend to construct prediction intervals that under-cover the response conditionally to some missing patterns. This motivates our novel generalized conformalized quantile regression framework, missing data augmentation, which yields prediction intervals that are valid conditionally to the patterns of missing values, despite their exponential number. We then show that a universally consistent quantile regression algorithm trained on the imputed data is Bayes optimal for the pinball risk, thus achieving valid coverage conditionally to any given data point. Moreover, we examine the case of a linear model, which demonstrates the importance of our proposal in overcoming the heteroskedasticity induced by missing values. Using synthetic and data from critical care, we corroborate our theory and report improved performance of our methods.

\end{abstract}
\end{adjustwidth}

\section{Introduction}
By leveraging increasingly large data sets, statistical algorithms and machine learning methods can be used to support high-stakes decision-making problems such as autonomous driving, medical or civic applications, and more.  
To ensure the safe deployment of predictive models  it is crucial to quantify the uncertainty of the resulting predictions, communicating the limits of predictive performance. Uncertainty quantification attracts a lot of attention in recent years, particularly methods that are based on  Conformal Prediction (CP) \citep{vovk_algorithmic_2005,papadopoulos_inductive_2002,lei_distribution-free_2018}. CP provides controlled predictive regions for any underlying predictive algorithm (e.g., neural networks and random forests), in finite samples with no assumption on the data distribution except for the exchangeability of the train and test data. More precisely, for a \textit{miscoverage rate} $\alpha \in [0,1]$,  CP outputs a \emph{marginally valid} prediction interval $\widehat{C}_\alpha$ for the test response $Y$ given its corresponding covariates $X$, that is:
\begin{equation}
\mathds{P}( Y \in \widehat{C}_\alpha (X) ) \geq 1 - \alpha.
\label{eq:cp_marg}
\end{equation}
Split CP \citep{papadopoulos_inductive_2002,lei_distribution-free_2018} ach\-ieves Eq. \eqref{eq:cp_marg} by keeping a hold-out set, the \textit{calibration~set}, used to evaluate the performance of a fixed predictive model. 

At the same time, as the volume of data increases, the volume of missing values also increases.
There is a vast literature on this topic \citep{Little2019,Josse2018StatScience}, and a recent survey even identified more than 150 different implementations \citep{mayerrmisstastic}. Missing values create additional challenges to the task of supervised learning, as traditional machine learning algorithms can not handle incomplete  data \citep{josse2019, lemorvan2020, lemorvan2020neumiss, lemorvan2021, ayme2022, vanness}. One of the most popular strategies to deal with missing values suggests imputing the missing entries with plausible values to get completed data, on which any analysis can be performed. The drawback of this ``impute-then-predict'' approach is that single imputation can distort the joint and marginal distribution of the data. Yet, \citet{josse2019, lemorvan2020, lemorvan2021} showed that such impute-then-predict strategies are Bayes consistent, under the assumption that a universally consistent learner is applied on an imputed data set. However, this line of work focuses on point prediction with missing values that aim to predict the most likely outcome. In contrast, our goal is quantifying predictive uncertainty, which was not explored with missing values although its enormous importance.

\begin{figure*}[!b]
    
    \vspace{-10pt}
    
    \centering

    \begin{subfigure}{0.49\textwidth}
    
    \begin{center}\includegraphics[width=\textwidth]{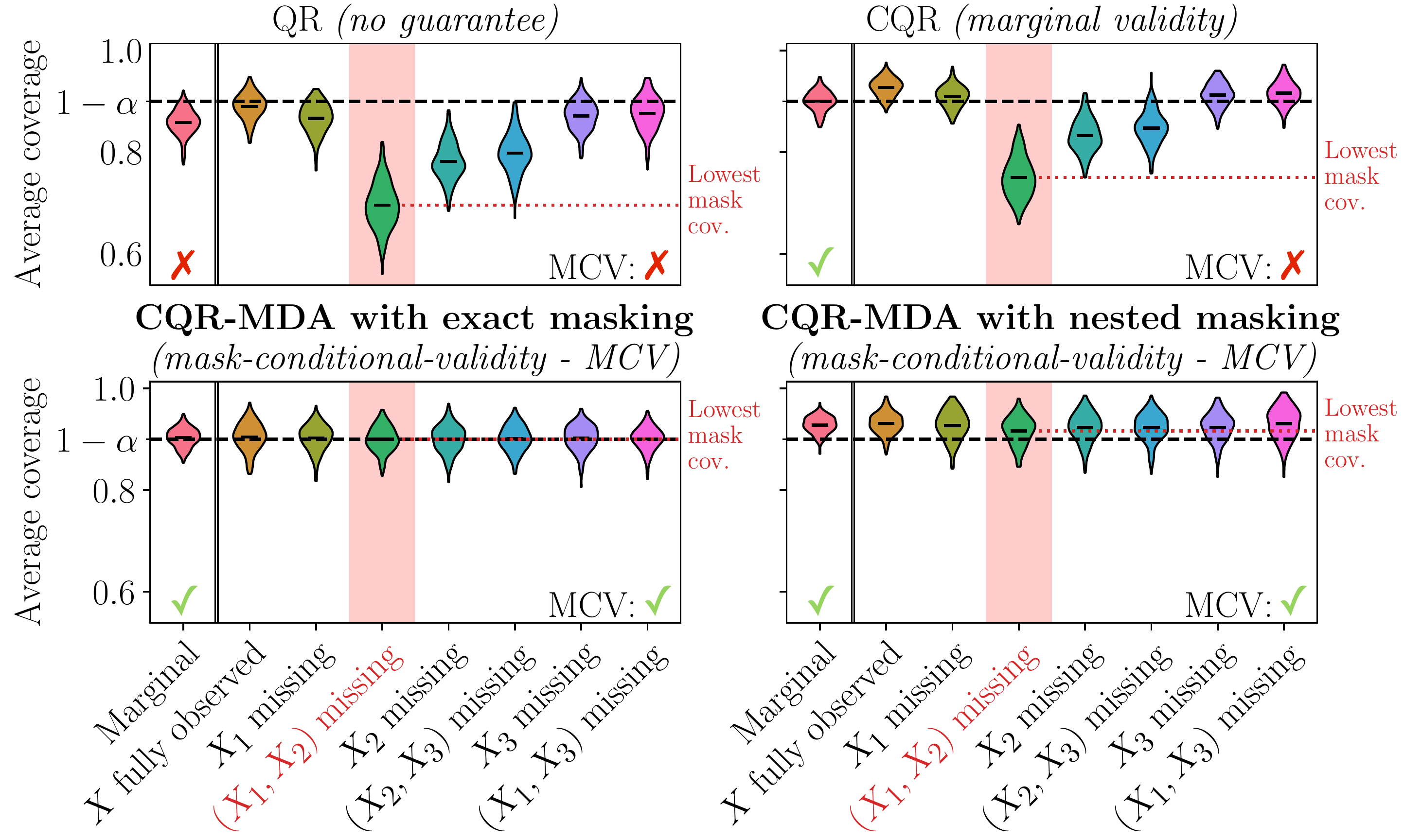}
    \end{center}
    \caption{
    Coverage of the predictive intervals depending on which features are missing, among the 3 features.  Evaluation over 200 runs.} 
    \label{fig:poc_toy}

    \end{subfigure}
    \hfill
    \begin{subfigure}{0.48\textwidth}
    \includegraphics[width=\textwidth]{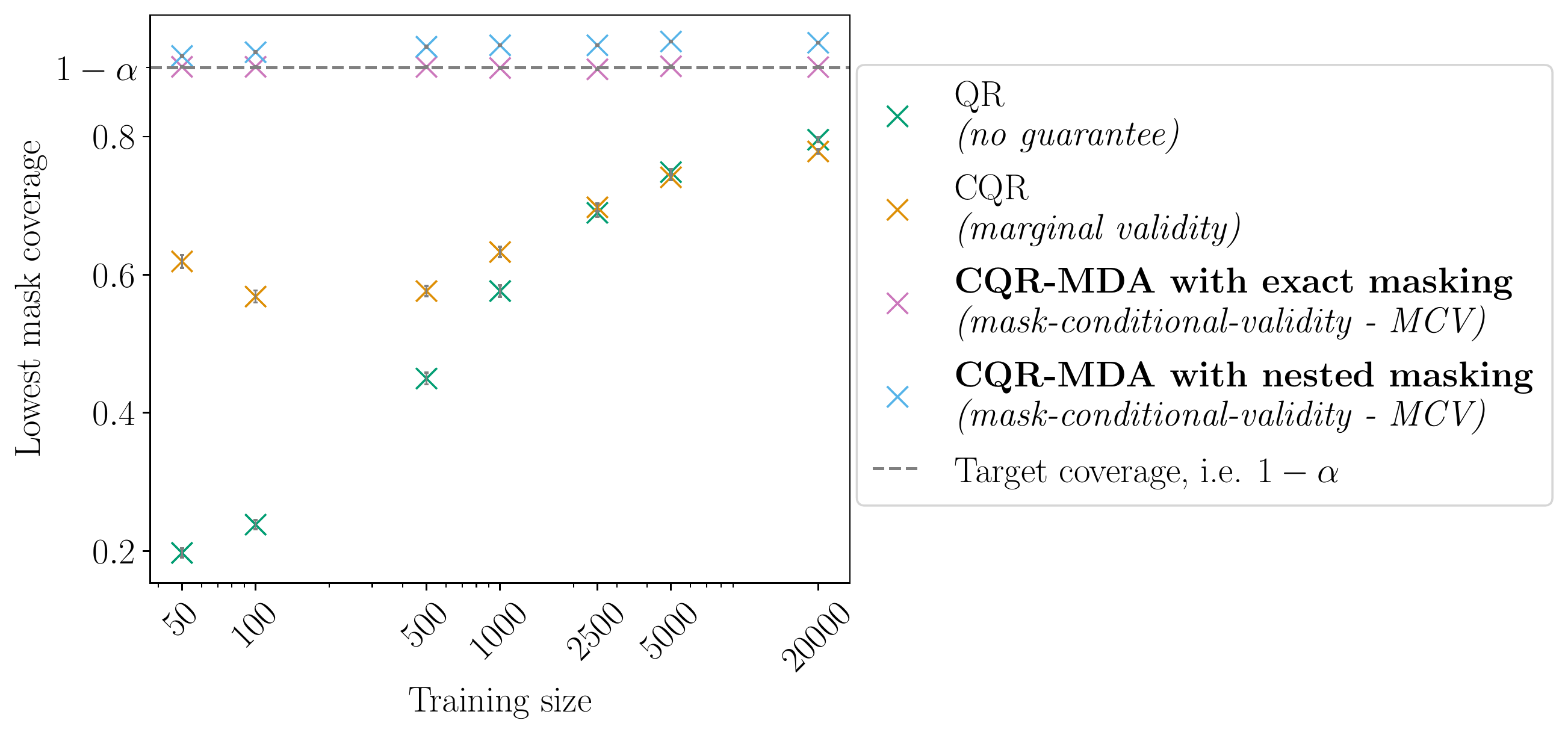}
    \caption{
    Lowest mask coverage as a function of the training size. Results evaluated over 100 repetitions, and the (tiny) error bars correspond to standard errors.
    }
    \label{fig:poc_masking}
    \end{subfigure}
\caption{Methods are Quantile Regression (QR), Conformalized Quantile Regression (CQR), and two novel procedures \textbf{\masksub} and \textbf{\mask}, on top of CQR. Settings are given in \Cref{sec:experiments}, in a nutshell: data follows a Gaussian linear model where missing values are independent of everything else and of proportion 20\%; the dimension of the problem is 3 in \Cref{fig:poc_toy} while in \ref{fig:poc_masking} it is 10.}
\label{fig:poc}
\end{figure*}

\subsection*{Contributions.}
We study CP with missing covariates. Specifically, we study downstream quantile regression (QR) based CP, like CQR \citep{romano_conformalized_2019}, on impute-then-predict~strategies. Still, the proposed approaches also encapsulate other regression basemodels, and even classification.

After setting background in \Cref{sec:framework}, our~first contribution is showing that CP on impute-then-predict~is \emph{marginally} valid regardless of the~model, missingness distribution, and imputation function (\Cref{sec:cp_na_marg}).

Then, we focus on the specificity of uncertainty quantification \emph{with missing values}. In Section~\ref{sec:th_glmdata}, we describe how different masks (i.e. the set of observed features) introduce additional heteroskedasticity: \textit{the uncertainty on the output strongly depends on the set of predictive features observed}. We therefore focus on achieving valid coverage \emph{conditionally on the mask}, coined MCV -- Mask-Conditional-Validity. MCV is desirable in practice, as occurrence of missing values are linked to important attributes (see \Cref{sec:methods}).

Traditional approaches such as QR and CQR fail to achieve MCV because they do not account for this core connection between missing values and uncertainty. This is illustrated on synthetic data in \Cref{fig:poc}. In \Cref{fig:poc_toy}, a toy example with only 3 features, thus $2^3-1=7$  possible masks, shows how the coverage of QR and CQR varies depending on the mask. Both methods dramatically undercover when the most important variable ($X_2$) is missing, and the loss of coverage worsens when additional features are missing. In particular, for each method, one mask ($X_1$ and $X_2$ missing, highlighted in \textcolor{tabred}{red}) leads to the \emph{lowest mask coverage}. Achieving MCV corresponds to a lowest mask coverage greater than $1-\alpha$. In \Cref{fig:poc_masking}, the dimension is 10: instead of the $2^{10}-1=1023$ different masks, we only report the lowest mask coverage for increasing sample sizes. It highlights that QR (\textcolor{blindgreen}{green $\times$}) and CQR (\textcolor{blindorange}{orange $\times$}) do not meet the lowest mask coverage target of 90\%, even for large sample sizes.

This motivates our second contribution: we show in \Cref{sec:methods} how to form prediction intervals that are MCV. This is highly challenging since there are exponentially many possible patterns to consider. Therefore, the naive solution to perform a calibration for each mask would fail as in finite samples, we often observe test samples with a mask that have low (or even null) frequency of appearance in the calibration set. To tackle this issue, we suggest two conformal methods that share the same core idea of missing data augmentation (MDA): the calibration data is artificially masked to match the mask of the point we consider at test time. The first method, \textit{CP-MDA with exact masking}, relies on building an ideal calibration set for which the data points have the exact same mask as of the test point.
We show its MCV under exchangeability and Missing Completely At Random assumptions. Our second method, \textit{CP-MDA with nested masking}, does not require such an ideal calibration set. Instead, we artificially construct a calibration set in which the data points have \textit{at least} the same mask as the test point, i.e., this artificial masking results in calibration points having possibly more missing values than the test point. We show the latter method also achieves the desired coverage conditional on the mask, but at the cost of an additional assumption for validity: stochastic domination of the quantiles. \Cref{fig:poc} illustrates those findings: both methods are MCV, as their lowest mask coverage is above $1-\alpha$.

Our third contribution further supports our design choice to use QR. We show that QR on impute-then-predict strategy is Bayes-consistent -- it can achieve the strongest form of coverage conditional on the observed test features (\Cref{sec:infinite}).
 
Lastly, we support our proposal using both (semi)-synthetic experiments and real medical data (\Cref{sec:experiments}). The code to reproduce our experiments is available on \href{ https://github.com/mzaffran/ConformalPredictionMissingValues}{GitHub}.

\section{Background}
\label{sec:framework}
\textbf{Background on missing values. } \label{sec:back_na} Consider a data set with $n$ exchangeable realizations of the random variable
$(X, M, Y) \in \mathds{R}^d \times \{0,1\}^d \times \mathds{R}$: $\left\{ \left(X^{(k)},M^{(k)},Y^{(k)} \right) \right\}_{k=1}^n$, where $X$ represents the features, $M$ the missing pattern, or mask,  and $Y$ an outcome to predict. 
For $j \in \llbracket 1,d \rrbracket$, $M_j = 0$ when $X_j$ is observed and $M_j = 1$ when $X_j$ is missing, i.e. \texttt{NA} (Not Available). We note $\mathcal{M} = \{0,1\}^d$ the set of masks.
For a pattern $m \in \mathcal{M}, X_{\text{obs}(m)}$ is the random vector of observed components, and $X_{\text{mis}(m)}$ is the random vector of unobserved ones. For example, if we observe $(\texttt{NA},6,2)$ then $m = (1,0,0)$ and $X_{\text{obs}(m)} = (6,2)$. Our goal is to predict a new outcome $Y^{(n+1)}$ given $X_{\text{obs}(M^{(n+1)})}^{(n+1)}$ and~$M^{(n+1)}$.
\begin{assumption}[exchangeability]
\label{ass:iid}
The random variables $\left( X^{(k)}, M^{(k)}, Y^{(k)} \right)_{k=1}^{n+1}$ are exchangeable.
\end{assumption}

Following \citet{rubin1976inference}, we consider three well-known missingness mechanisms.

\begin{definition}[Missing Completely At Random (MCAR)]
For any $m \in \mathcal{M}$,  $\mathds{P}\left(M = m | X\right) = \mathds{P}\left(M = m\right)$.
\end{definition}

\begin{definition}[Missing At Random (MAR)]
For any $m \in \mathcal{M}$, $\mathds{P}\left(M = m | X\right) = \mathds{P}\left(M = m | X_{\text{obs}(m)}\right)$.
\end{definition}

\begin{definition}[Missing Non At Random (MNAR)]
If the missing data is not  MAR, it is MNAR. Thus, its probability distribution depends on $X$, including the missing values.
\end{definition}

\textbf{Impute-then-predict.} As most predictive algorithms can not directly handle missing values, we impute the incomplete data using an imputation function $\Phi$ which maps observed values to themselves and missing values to a function of the observed values. With notations from \citet{lemorvan2021} we note $\phi^{m} : \mathds{R}^{|\obs(m)|} \rightarrow \mathds{R}^{|\mis(m)|}$ the imputation function which takes as input observed values and outputs imputed values, i.e. plausible values, given a mask $m \in \mathcal{M}$. Then, the imputation function $\Phi$ belongs to 
$\begin{aligned}
\mathcal{F}^I := & \left\{ \right. \Phi : \mathds{R}^d \times \mathcal{M} \rightarrow \mathds{R}^d : \forall j \in \llbracket 1,d \rrbracket, \\
& \left. \Phi_j \left( X, M \right) = X_j\mathds{1}_{M_j = 0} + \phi_j^{M}\left(X_{\text{obs}(M)}\right) \mathds{1}_{M_j = 1}  \right\}.
\end{aligned}$ 
Additionally, $\mathcal{F}^{I}_{\infty}$ is the restriction of $\mathcal{F}^{I}$ to $\mathcal{C}^{\infty}$ functions which include deterministic imputation, such as mean imputation or imputation by regression. The imputed data set is formed by the realizations of the $n$ random variables $\left(\Phi\left(X, M\right), M, Y\right)$.
In practice, $\Phi$ is obtained as the result of an algorithm $\mathcal{I}$ trained on $\left\{ \left(X^{(k)},M^{(k)} \right) \right\}_{k=1}^{n+1}$.

\begin{assumption}[Symmetrical imputation]
\label{ass:imp_sym}
The imputation function $\Phi$ is the output of an algorithm $\mathcal{I}$ treating its input data points symmetrically: $\mathcal{I}(( X^{(\sigma(k))}, M^{(\sigma(k))} )_{k=1}^{n+1}) \overset{(d)}{=} \mathcal{I}(( X^{(k)}, M^{(k)} )_{k=1}^{n+1})$
\textit{conditionally on}  $( X^{(k)}, M^{(k)} )_{k=1}^{n+1}$ and for any permutation $\sigma$ on~$\llbracket1,n+1\rrbracket$.
\end{assumption}

\Cref{ass:imp_sym} is very mild and satisfied by all existing imputation methods for exchangeable data. In particular, it is valid for iterative regression imputation which allows out-of-sample imputation. 

\textbf{Background on (split) conformal prediction.} Split, or inductive, CP (SCP) \citep{papadopoulos_inductive_2002,lei_distribution-free_2018} builds predictive regions  by first splitting the $n$ points of the training set into two disjoint sets $\rm{Tr}, \rm{Cal}\subset \llbracket 1,n \rrbracket$, to create a \textit{proper training set}, $\rm{Tr}$, and a \textit{calibration set}, $\rm{Cal}$. On the proper training set, a model $\hat f$ (chosen by the user) is fitted, and then used to predict on the calibration set. \textit{Conformity scores} 
$S_{\rm{Cal}}=  \{(s(X^{(k)},Y^{(k)}))_{k \in \rm{Cal}}\}$
are computed to assess how well the fitted model $\hat f$ predicts the response values of the calibration points. 
For example, Conformalized Quantile Regression \citep[CQR,][]{romano_conformalized_2019} fits two quantile regressions $\hat{q}_{\text{low}}$ and $\hat{q}_{\text{upp}}$, on the proper training set. The conformity scores are defined by $s(x,y) = \max(\hat{q}_{\text{low}}(x) - y, y - \hat{q}_{\text{upp}}(x))$. Finally, a corrected $(1-\tilde\alpha)$-th quantile  of these scores $\widehat{Q}_{1-\tilde\alpha}(S_{\rm{Cal}})$ is computed (called \textit{correction term}) to define the predictive region: 
$\widehat{C}_{\alpha} ( x ) := \{y \text{ such that } s( y, \hat{f} ( x ) ) \leq \widehat{Q}_{1-\tilde\alpha}(S_{\rm{Cal}})  \}$.\footnote{The correction $\alpha \to \tilde\alpha$ is needed because of the inflation of quantiles in finite sample (see Lemma 2 in \citet{romano_conformalized_2019} or Section 2 in \citet{lei_distribution-free_2018}).} An illustration of CQR is provided in \Cref{app:cqr}.

This procedure satisfies Eq. \eqref{eq:cp_marg} for any $\hat{f}$, any (finite) sample size $n$, as long as the data points are exchangeable.\footnote{Only the calibration and test data points need to be exchangeable.} Moreover, if the scores are almost surely distinct, the coverage holds almost exactly:
$\mathds{P}(Y \in \widehat C_{\alpha} (X)) \leq 1-\alpha+\frac{1}{\#\rm{Cal}+1}$.

For more details on SCP, we refer to \citet{angelopoulos-gentle,vovk_algorithmic_2005}, as well as to \citet{manokhin_valery_2022_6467205}.

\section{Warm-up: marginal coverage with \texttt{NA}s}
\label{sec:cp_na_marg}

A first idea to get valid predictive intervals $\widehat C_\alpha (X,M)$ in the presence of missing values $M$ is to apply CP in combination with impute-then-predict, which we refer to as \textit{impute-then-predict+conformalization}. More details on this approach are given in \Cref{app:itp+conf_alg} for both classification and regression tasks, although our main focus is regression.
It turns out that such a simple approach is marginally (exactly) valid. 

\begin{definition}[Marginal validity]
\label{def:marg_cov}
A method outputting intervals $\widehat{C}_{\alpha}$ is marginally valid if the \camerareadyrevisionlast{following} lower bound is satisfied, and exactly valid if the \camerareadyrevisionlast{following} upper bound is also satisfied:
\begin{align*}
1-\alpha \underset{\text{validity}}{\leq} \mathds{P}\left(  Y ^{(n+1)} \in \widehat{C}_{\alpha} \left( X^{(n+1)},M^{(n+1)} \right) \right) \\[-10pt]
\underset{\text{exact validity}}{\leq} 1 - \alpha + \frac{1}{\#\rm{Cal} + 1}.
\end{align*}

\end{definition}

Indeed, symmetric imputation preserves exchangeability.

\begin{lemma}[Imputation preserves exchangeability]
\label{lem:exch_imp}

Let \ref{ass:iid} hold. Then, for any missing mechanism, for any imputation function $\Phi$ satisfying \ref{ass:imp_sym}, the imputed random variables $\left(\Phi\left(X^{(k)},M^{(k)} \right),M^{(k)},Y^{(k)}\right)_{k=1}^{n+1}$ are exchangeable. 
\end{lemma}

Note that if we replace \ref{ass:iid} by an i.i.d. assumption, the imputed data set is only exchangeable but not i.i.d. without further assumptions on $\mathcal{I}$. Indeed, even simple mean imputation breaks independence.

\begin{proposition}[(Exact) validity of impute-then-predict+conformalization]
\label{prop:marg_cp_na}
If \ref{ass:iid} and \ref{ass:imp_sym} are satisfied, impute-then-predict+conformalization is marginally valid. If moreover the scores are almost surely distinct, it is exactly valid.
\end{proposition}

This is an important first positive result (proved in \Cref{app:itp+conf_proof}) showing that CP applied on an imputed data set has the same validity properties as on complete data, regardless of the missing value mechanism (MCAR, MAR or MNAR) and of the symmetric imputation scheme. Note that similar propositions could be derived for full CP \citep{vovk_algorithmic_2005} and Jackknife+ \citep{barber2021jackknife}. 

\Cref{prop:marg_cp_na} complements the work by \citet{yang2015features}, that also guarantees \emph{marginal} coverage for full CP, with the striking difference of having a complete training data.

\section{Challenge: \texttt{NA}s induce heteroskedasticity}
\label{sec:th_glmdata}
To better understand  the interplay between missing values and conditional coverage with respect to the mask, we consider an illustrative example of a Gaussian linear model. 

\begin{model}[Gaussian linear model]
\label{mod:glm}

The data is generated according to a linear model and the covariates are Gaussian conditionally to the pattern:
\begin{itemize}[topsep=0pt,noitemsep,leftmargin=*]
	\item $Y = \beta^T X + \varepsilon$, $\varepsilon \sim \mathcal{N}(0, \sigma^2_{\varepsilon}) \perp\!\!\!\!\perp (X,M)$, $\beta \in \mathds{R}^d$.
	\item for all $m \in \mathcal{M}$, there exist $\mu^{m}$ and $\Sigma^{m}$ such that
   $ X | (M = m) \sim \mathcal{N}(\mu^{m}, \Sigma^{m}).$
\end{itemize}

\end{model}

In particular, \Cref{mod:glm} is verified when $X$ is Gaussian and the missing data is MCAR. \Cref{mod:glm} is more general: it even includes MNAR examples \citep{ayme2022}.

\begin{proposition}[Oracle intervals]
\label{prop:oracle_intervals}

The oracle predictive interval is defined as the smallest valid interval knowing $X_{\obs(M)}$ and $M$. Under \Cref{mod:glm}, its length only depends on the mask. For any $m \in \mathcal{M}$ this oracle length is:
\begin{equation}
\mathcal{L}^*_{\alpha}(m) = 2 q^{\mathcal{N}(0,1)}_{1-\frac{\alpha}{2}} \sqrt{\beta^T_{\mis(m)} \Sigma^{m}_{\mis|\obs} \beta_{\mis(m)} + \sigma^2_{\varepsilon}}.
\label{eq:oracle_glm}
\end{equation}
See \Cref{app:glm} for the definition of $\mu^{m}_{\mis|\obs}$ and $\Sigma^{m}_{\mis|\obs}$ and the quantiles of $Y | (X_{\obs(m)}, M = m)$.
\end{proposition}

Eq. \eqref{eq:oracle_glm} stresses that even when the noise of the generative model is~homoskedastic,  \textit{missing values induce heteroskedasticity}. Indeed, the covariance of the conditional distribution of $Y | (X_{\obs(m)}, M = m)$ depends on $m$. Furthermore, the uncertainty increases when missing values are associated with larger regression coefficients (i.e. the most predictive variables): if $\beta_{\mis(m)}$ is large, then $\mathcal{L}^*_{\alpha}(m)$ is also large, as $\Sigma^{m}_{\mis|\obs}$ is positive. In the extreme case where all the variables are missing, i.e. $m = (1,\cdots,1)$, $\mathcal{L}^*_{\alpha}(m) = 2 q^{\mathcal{N}(0,1)}_{1-\frac{\alpha}{2}} \sqrt{ \beta \Sigma^{m} \beta^T + \sigma_\varepsilon^2} = q^{Y}_{1-\frac{\alpha}{2}} - q^{Y}_{\frac{\alpha}{2}}$. On the contrary, if $m = (0,\cdots,0)$ (that is all $X_j$ are observed), $\beta_{\mis(m)}$ is empty and $\mathcal{L}^*_{\alpha}(m) = 2 q^{\mathcal{N}(0,1)}_{1-\frac{\alpha}{2}} \sigma_\varepsilon = q^{\varepsilon}_{1-\frac{\alpha}{2}} - q^{\varepsilon}_{\frac{\alpha}{2}}$. We illustrate this induced heteroskedasticity and the impact of the predictive power in \Cref{fig:poc_toy}, and in \Cref{app:glm} along with a discussion emphasizing that even with the Bayes predictor for the conditional mean, mean-based CP does not yield intervals that are MCV. 

The above analysis motivates the following two design choices we make in this work. First, we advocate working with QR models rather than classic regression ones, as the former can handle heteroskedastic data. Second, we recommend providing the mask information to the model in addition to the input covariates, as the mask may further encourage the model to construct an interval with a length adaptive to the given mask. 
Therefore, we focus on CQR \citep{romano_conformalized_2019}\footnote{Note that our proposed framework is not based on CQR, this is only one instance of it.}, an adaptive version of SCP, and concatenate the mask to the features. However, the predictive intervals of this procedure may not necessarily provide valid coverage conditionally on the masks, especially in finite samples as shown in \Cref{fig:poc_masking} (\textcolor{blindorange}{orange crosses}). This is because the quality of the prediction at some $(X,M)$ depends strongly on $M$, as there is an exponential number of patterns ($2^d$) for a finite training size, whereas the correction term is calculated independently of the masks.

\section{Achieving mask-conditional-validity \camerareadyrevision{(MCV)}}
\label{sec:methods}
We now aim at achieving \textit{mask-conditional-validity} (MCV) defined as follows using an ordering on the masks.
\begin{definition}[Included masks] 
\label{def:included_masks}
Let $(\mathring{m},\breve{m}) \in \mathcal{M} ^2$,  $\mathring{m} \subset \breve{m}$ if for any $j \in \llbracket 1,d \rrbracket$ such that $\mathring{m}_j = 1$ then $\breve{m}_j = 1$, i.e. $\breve{m}$ includes at least the same missing values than $\mathring{m}$.
\end{definition}

\begin{definition}[MCV]
\label{def:cond_cov}
A method is MCV if for any $m \in \mathcal{M}$ the followinglower bound is satisfied, and exactly MCV if for any $m \in \mathcal{M}$ the followingupper bound is also satisfied:
\begin{align*}
1-\alpha \!\!\! \underset{\text{valid}}{\leq} \!\! \mathds{P}\left(  Y ^{(n+1)} \in \widehat{C}_{\alpha} \left( X^{(n+1)},m \right) | M^{(n+1)} = m \right) \\[-10pt]
\underset{\text{exactly valid}}{\leq} 1 - \alpha + \frac{1}{\#\rm{Cal}^{m} + 1},
\end{align*}
where $\rm{Cal}^{m} = \left\{ k \in \rm{Cal} \text{ such that } m^{(k)} \subset m \right\}$.
\end{definition}

\begin{figure}[!b]
    \centering
    \centerline{\includegraphics[width=0.55\textwidth]{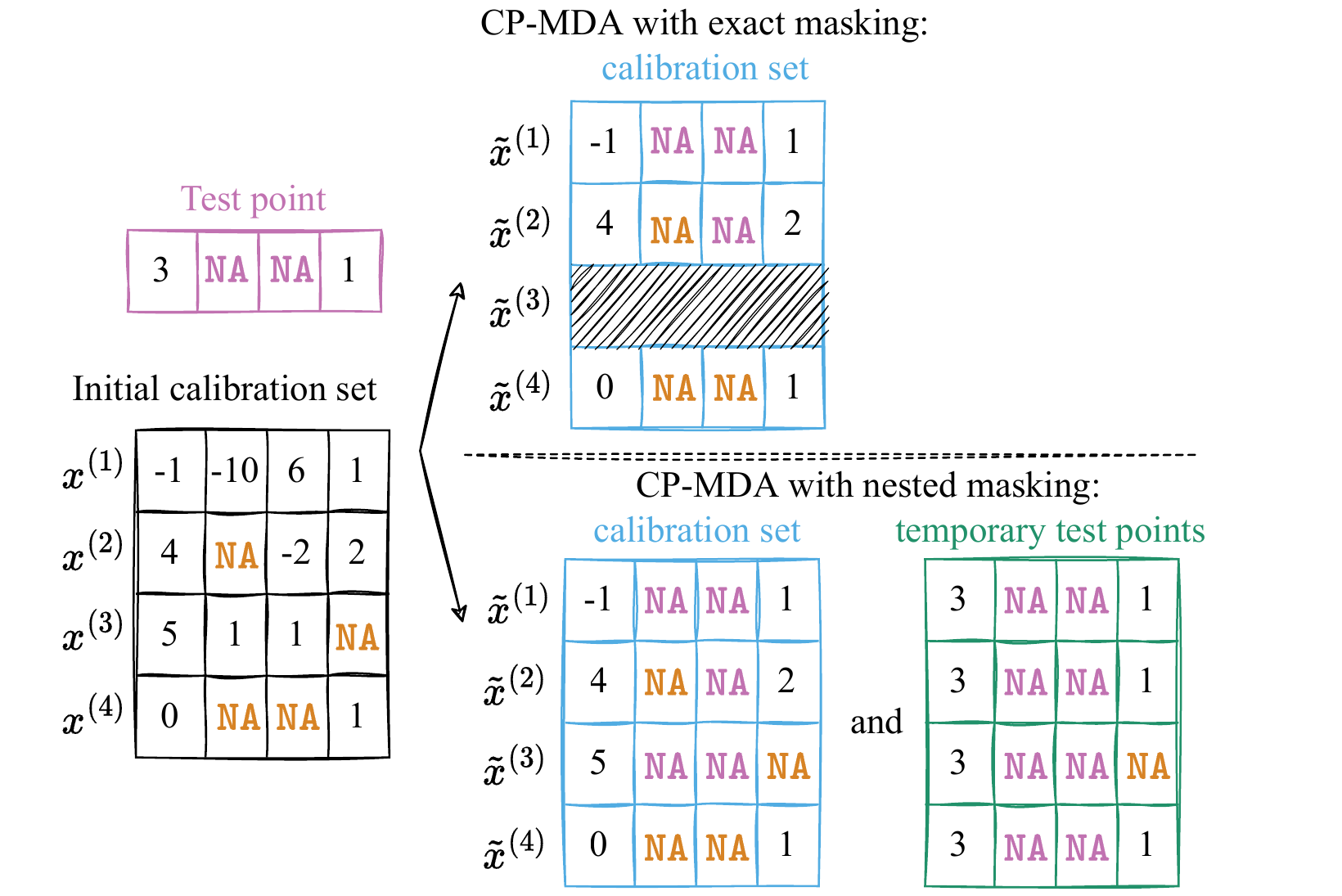}}
    \vspace{5pt}
    \caption{CP-MDA illustration. \textcolor{blindblue}{Augmented calibration set} according to one \textcolor{blindpurple}{test point}. For \mask, the \textcolor{blindblue}{augmented masks of the calibration set} are also \textcolor{blindgreen}{applied temporarily to the test point}.}
    \label{fig:algos_scheme}
\end{figure}

\paragraph{On the relevance of MCV.} In a medical application context, it is very common to have missing data completely at random (MCAR) when a measurement device fails or  the medical team forgot to fill out some forms. As a general rule, from an \emph{equity standpoint}, a patient whose data is missing should not be penalized (because of ``bad luck'') by being assigned a prediction interval that is less likely to include the true response than if the data were complete. 

Furthermore, the mask can also be linked to an external unobserved feature corresponding to a meaningful category. Consider the problem of predicting a disease among a population. Aggregating data from multiple hospitals with different practices and measurement devices can imply different features are observed for each patient. This can be viewed as a MCAR setting when \textit{identically distributed} patients\footnote{say, for example young children whose input/output distribution is \textit{not} dependent on the neighborhood.} are assigned an hospital at random. Patterns are then linked to the cities, that themselves are related to socio-economical data. 

Overall, the missing patterns form \textit{meaningful categories} and \textit{ensuring MCV yields more equitable treatment}. 
Therefore, a method achieving marginal coverage by systematically failing on a given pattern, even in a MCAR setting, is not suitable. Finally, in non-MCAR cases, the pattern may be exactly related to critical discriminating features.

\subsection{Missing Data Augmentation (MDA)}
\label{sec:methods_desc}

To obtain a MCV procedure, we suggest \textcolor{blindblue}{modifying the calibration set} according to the \textcolor{blindpurple}{mask of the test point}, while the training step is unchanged. More precisely, the mask of the test point is applied to the calibration set, as illustrated in \Cref{fig:algos_scheme}. The rationale is to mimic the missing pattern of the test point by artificially augmenting the calibration set with that mask. It ensures that the correction term is computed using data with (at least) the same missing values as the test point. We refer to this strategy as \textit{CP with Missing Data Augmentation} (CP-MDA), and derive two versions of it. \Cref{alg:cp_na_sub,alg:cp_na_jack} are written using CQR as the base conformal procedure, but they work with any conformal method as we describe in \Cref{app:methods_general}. \newline

\textbf{\Cref{alg:cp_na_sub} -- \masksub.} CP-MDA with \textit{exact masking} consists of keeping the \emph{artificially} \textcolor{blindblue}{masked calibration points} (l.~\ref{line_algo_1:mask_points_incal}) that have exactly the same missing pattern as the \textcolor{blindpurple}{test point} (l.~\ref{line_algo_1:compute_Caltest}). 
Then \Cref{alg:cp_na_sub} performs as impute-then-predict+conformalization: impute the calibration set (l.~\ref{line_algo_1:imput_cal}), predict on it and get the calibration scores (l.~\ref{line_algo_1:computescore}), compute their quantile to obtain the correction term (l.~\ref{line_algo_1:computequantile}), and finally impute and predict the test point with the fixed fitted model by adding and subtracting the correction term (l.~\ref{line_algo_1:compute_interval}) to the initial conditional quantile estimates. Note that \Cref{alg:cp_na_sub} is described for one test point for simplicity but extends easily to many test points. The computations are then shared: 
the training part (l.~\ref{line_algo_1:split}-\ref{line_algo_1:fit}) is common to any test point and the correction term (l.~\ref{line_algo_1:compute_Caltest}-\ref{line_algo_1:computequantile}) can be reused for any new test point with the same mask.

\begin{figure}[!t]
\vspace{-12pt}
\begin{algorithm}[H]
\caption{\masksub (with CQR)}
\label{alg:cp_na_sub}
\begin{algorithmic}[1] 
\REQUIRE Imputation algorithm $\mathcal{I}$, quantile regression algorithm $\mathcal{QR}$, significance level $\alpha$, training set $\left\{ \left(x^{(k)},m^{(k)},y^{(k)}\right)\right\}_{k=1}^n$, test point $\left( \textcolor{blindpurple}{ x^{(\text{test})},m^{(\text{test})} } \right)$
\ENSURE Prediction interval ${\widehat{C}}_{\alpha}\left( x^{(\text{test})},m^{(\text{test})} \right)$
\STATE \label{line_algo_1:split}Randomly split $\{1, \ldots, n\}$ into 2 disjoint sets $\rm{Tr}$ \& $\rm{Cal}$
\STATE Fit the imputation function: ${\Phi(\cdot) \leftarrow \mathcal{I}\left(\left\{ \left( x^{(k)}, m^{(k)} \right), k \in \rm{Tr}\right\}\right)}$ 
\STATE Impute the training set: ${\forall k \in \rm{Tr}}, x_{\text{imp}}^{(k)} = \Phi(x^{(k)},m^{(k)})$
\STATE \label{line_algo_1:fit} Fit $\mathcal{QR}$: 
\vspace{-8pt}
\begin{align*}
\hat{q}_{\frac{\alpha}{2}}(\cdot) & \leftarrow \mathcal{QR}\left(\left\{\left( x_{\text{imp}}^{(k)}, y^{(k)} \right), k \in \rm{Tr}\right\}, \alpha/2 \right) \\
\hat{q}_{1-\frac{\alpha}{2}}(\cdot) & \leftarrow \mathcal{QR}\left(\left\{\left( x_{\text{imp}}^{(k)}, y^{(k)} \right), k \in \rm{Tr}\right\}, 1-\alpha/2 \right)
\end{align*}  

\STATEx \textcolor{blindblue}{// Generate an augmented calibration set:} 
\STATE \label{line_algo_1:compute_Caltest}$\rm{Cal^{(\text{test})}} = \left\{ k \in \rm{Cal} \text{ such that } m^{(k)} \subset m^{(\text{test})} \right\}$
\FOR {$k \in \rm{Cal^{(\text{test})}}$}
\STATE \label{line_algo_1:mask_points_incal}$\widetilde m^{(k)} = \textcolor{blindpurple}{m^{(\text{test})}}$ \COMMENT {//Additional masking}

\ENDFOR \COMMENT{\textcolor{blindblue}{\;\; Augmented calibration set generated. //} }
\FOR {$k \in \rm{Cal^{(\text{test})}} $} 
\STATE \label{line_algo_1:imput_cal} Impute the calibration set: 
$x_{\text{imp}}^{(k)}=\Phi(x^{(k)},\widetilde m^{(k)})$
\STATE \label{line_algo_1:computescore}Set ${s^{(k)} = \max( \hat{q}_{\frac{\alpha}{2}}(x_{\text{imp}}^{(k)}) - y^{(k)} , y^{(k)} - \hat{q}_{1-\frac{\alpha}{2}}(x_{\text{imp}}^{(k)}) )}$
\ENDFOR
\STATE \label{line_algo_1:compute_score_set}Set $S = \{s^{(k)}, k \in \rm{Cal^{(\text{test})}}\}$
\STATE \label{line_algo_1:computequantile}Compute $\widehat Q_{1-\tilde\alpha}\left(S\right)$, the $1-\tilde\alpha$-th empirical quantile of $S$, with $1-\tilde\alpha := (1-\alpha)\left(1+{1}/{\#S}\right)$
\STATE \label{line_algo_1:compute_interval}Set  $\widehat{C}_{\alpha} ( \textcolor{blindpurple}{ x^{(\text{test})}, m^{(\text{test})} } ) =  \left[ \right. \hat{q}_{\frac{\alpha}{2}} \circ \Phi ( x^{(\text{test})}, m^{(\text{test})} ) - \widehat Q_{1-\tilde\alpha}\left(S\right) ; \left. \hat{q}_{1-\frac{\alpha}{2}} \circ \Phi ( x^{(\text{test})}, m^{(\text{test})} ) + \widehat Q_{1-\tilde\alpha}\left(S\right) \right]$
\end{algorithmic}
\end{algorithm}
\vspace{-20pt}
\end{figure}

In high dimensions, many calibration points may be discarded when applying \masksub since it is likely that their missing patterns would not be included in the one of the test point.\footnote{Yet, these discarded points could be used for training but this comes at the cost of fitting a different model for each pattern; such a path is reasonable if the data is scarce.} This limitation brings us to the second algorithm we propose, \mask.

\begin{figure}[t]
\vspace{-12pt}
\begin{algorithm}[H]
\caption{\mask (with CQR)}
\label{alg:cp_na_jack}
\begin{algorithmic}[1] 
\REQUIRE Same as \Cref{alg:cp_na_sub}
\ENSURE Same as \Cref{alg:cp_na_sub}
\STATE Compute lines \ref{line_algo_1:split} to \ref{line_algo_1:fit} of Algorithm~\ref{alg:cp_na_sub}
\STATEx \textcolor{blindblue}{// Generate an augmented calibration set:}  
\FOR {$k \in \rm{Cal}$}
\COMMENT {Additional nested masking}
\STATE \label{line_algo_2:mask_all}$\widetilde m^{(k)} = \max( \textcolor{blindpurple}{m^{(\text{test})}}, m^{(k)})$ 
\ENDFOR \COMMENT{\textcolor{blindblue}{\;\; Augmented calibration set generated. //} }
\FOR {$k \in \rm{Cal} $} 
\STATE \label{line_algo_2:impute_cal}Impute the calibration set: $x_{\text{imp}}^{(k)}:= \Phi\left(x^{(k)},\widetilde m^{(k)}\right)$
\STATE \label{line_algo_2:compute_score}Set ${s^{(k)} = \max( \hat{q}_{\frac{\alpha}{2}}(x_{\text{imp}}^{(k)}) - y^{(k)} , y^{(k)} - \hat{q}_{1-\frac{\alpha}{2}}(x_{\text{imp}}^{(k)}) )}$
\STATE \label{line_algo_2:pred_low}Set $z^{(k)}_{\frac{\alpha}{2}} = \hat{q}_{\frac{\alpha}{2}} \circ \Phi\left( \textcolor{blindgreen}{ x^{(\text{test})},\widetilde m^{(k)} }\right) - s^{(k)} $
\STATE \label{line_algo_2:pred_up}Set $z^{(k)}_{1-\frac{\alpha}{2}} = \hat{q}_{1-\frac{\alpha}{2}}\circ \Phi\left( \textcolor{blindgreen}{  x^{(\text{test})},\widetilde m^{(k)} } \right) + s^{(k)} $
\ENDFOR
\STATE \label{line_algo_2:bag_low}Set $Z_{\frac{\alpha}{2}} = \{z^{(k)}_{\frac{\alpha}{2}}, k \in \rm{Cal}\}$
\STATE \label{line_algo_2:bag_up}Set $Z_{1-\frac{\alpha}{2}} = \{z^{(k)}_{1-\frac{\alpha}{2}}, k \in \rm{Cal}\}$
\STATE \label{line_algo_2:quant_low}Compute $\widehat Q_{\tilde\alpha}\left(Z_{\frac{\alpha}{2}}\right)$
\STATE \label{line_algo_2:quant_up}Compute $\widehat Q_{1-\tilde\alpha}\left(Z_{1-\frac{\alpha}{2}}\right)$
\STATE \label{line_algo_2:compute_int}Set  $ {\widehat{C}}_{\alpha} \left( \textcolor{blindpurple}{ x^{(\text{test})}, m^{(\text{test})} } \right) = [ \widehat Q_{\tilde\alpha}\left(Z_{\frac{\alpha}{2}}\right) ; \widehat Q_{1-\tilde\alpha}\left(Z_{1-\frac{\alpha}{2}}\right) ]$
\end{algorithmic}
\end{algorithm}
\vspace{-20pt}
\end{figure}

\textbf{\Cref{alg:cp_na_jack} -- \mask.} CP-MDA with \textit{nested masking} avoids the removal of calibration points whose masks are not included in that of the test point. Instead, we \textcolor{blindblue}{apply} the \textcolor{blindpurple}{mask of the test point} to the \textcolor{blindblue}{calibration points}, and so \textcolor{blindblue}{we keep all the observations} (l.~\ref{line_algo_2:mask_all}). Next, we impute the \textcolor{blindblue}{masked calibration points} (l.~\ref{line_algo_2:impute_cal}) before computing their scores $s^{(k)}$ (l.~\ref{line_algo_2:compute_score}). Then, for each calibration point, the fitted quantile regressors are used to predict on the {\textcolor{blindgreen}{test point with a temporary mask}}, which matches \textcolor{blindgreen}{the mask of the given augmented calibration point}. These predictions are corrected with the score of the calibration point (l.~\ref{line_algo_2:pred_low}-\ref{line_algo_2:pred_up}) and stored in two bags $Z_{\frac{\alpha}{2}}$ for the lower interval boundary, and $Z_{1-\frac{\alpha}{2}}$ for the upper interval boundary (l.~\ref{line_algo_2:bag_low}-\ref{line_algo_2:bag_up}). The prediction is finally obtained by taking the $\alpha$ quantiles of the bags $Z$ (l.~\ref{line_algo_2:quant_low}-\ref{line_algo_2:compute_int}). 

The rationale for predicting on temporary test points with the mask of a given augmented calibration point is that we want to treat the test and calibration points in the same way.\footnote{This motivation is similar to the one of Jackknife+ \citep{barber2021jackknife} and out-of-bags methods \citep{gupta}.} We should note that this method may tend to achieve conservative coverage, since the augmented calibration set may have masks that overly include the missing pattern of the test point, i.e., the augmented points may have more missing values than the test point.

\subsection{Theoretical guarantees in finite sample}
\label{sec:methods_guar}

Let us consider the following assumptions. 
\begin{assumption}[$Y$ is not explained by $M$]
\label{ass:y_ind_m}
$( Y \ind M ) | X$.
\end{assumption}

\begin{assumption}[Stochastic domination of the quantiles]
\label{ass:sto_dom}
Let $(\mathring{m}, \breve{m}) \in \mathcal{M}^2$. If $\mathring{m} \subset \breve{m}$ then for any $\delta \in [0,0.5]$: 
\begin{itemize}[itemsep=0.5pt,topsep=0pt]
   \item  $q_{1-\delta/2}^{Y | (X_{\text{obs}\left(\mathring{m}\right)}, M = \mathring{m})} \leq q_{1-\delta/2}^{Y | (X_{\text{obs}\left(\breve{m}\right)}, M = \breve{m})}$,
    \item $q_{\delta/2}^{Y | (X_{\text{obs}\left(\mathring{m}\right)}, M = \mathring{m})} \geq q_{\delta/2}^{Y | (X_{\text{obs}\left(\breve{m}\right)}, M = \breve{m})}$.
\end{itemize}

\end{assumption}
\ref{ass:sto_dom} grasps the underlying intuition that the conditional distribution of $Y|(X_{\obs(m)}, M = {m})$ tends to have larger deviations when the number of observed variables is smaller, in concordance with the intuition that  observing predictive variables reduce the conditional randomness of~$Y|X_{\obs}$.

The following theorems (proved in \Cref{app:methods}) state the finite sample guarantees of CP-MDA.

\begin{theorem}[MCV of CP-MDA]
\label{prop:meth_cond}
Assume the missing mechanism is MCAR, and \ref{ass:iid} to \ref{ass:y_ind_m}. Then:
\begin{enumerate}[topsep=0pt,noitemsep,leftmargin=*,wide]
	\item \masksub is MCV;
	\item if the scores are almost surely distinct, \masksub is exactly MCV;
	\item if \ref{ass:sto_dom} also holds, \mask is MCV, up to a technical minor modification of the output. 
\end{enumerate}
\end{theorem}
The challenge in proving MCV of \mask is that the augmented calibration and test points are not exchangeable conditional on the mask and thus may result in under-coverage. 
However, by imposing \ref{ass:sto_dom} we prove that this violation of exchangeability still leads to MCV (and often conservative MCV) (see Lemma~\ref{lem:aux}).  
We conjecture that \mask attains MCV (without any modification), as also supported by  experiments. However, we~could not prove it without making an independence assumption which we prefer to avoid as exchangeability is key to imputation methods. Instead, we prove in \Cref{th:precise_conservative} the MCV of any variant outputting  $[ \widehat Q_{\tilde\alpha}(Z^{\tilde{m}}_{\frac{\alpha}{2}}) ; \widehat Q_{1-\tilde\alpha}(Z^{\tilde m}_{1-\frac{\alpha}{2}}) ]$ for $Z^{\tilde{m}}_{\frac{\alpha}{2}} $ the subset of  $Z_{\frac{\alpha}{2}}$  composed with points using mask $\tilde m$ at l.~\ref{line_algo_2:impute_cal}-\ref{line_algo_2:pred_up}.

\begin{theorem}[Marginal validity of CP-MDA]
\label{prop:meth_marg}
Under then same assumptions as \Cref{prop:meth_cond} (i) \masksub is marginally valid; (ii) if \ref{ass:sto_dom} also holds, \mask is marginally valid (with the same caveats as in \Cref{prop:meth_cond}).

\end{theorem}

\section{Towards asymptotic individualized coverage}
\label{sec:infinite}
Achieving validity conditionally on the mask is an important step towards conditional coverage: in practice one aims at the strongest coverage conditional on \textit{both} $X$ and $M$. \citet{lei_distribution-free_2014,vovk_conditional_2012,barber_limits_2021} studied a related question (without considering missing patterns) and concluded that it is impossible to achieve \textit{informative} intervals satisfying conditional coverage, $\mathds{P}( Y \in \widehat{C}_\alpha (x) | X = x ) \geq 1 - \alpha$ for any $x \in \mathcal{X}$ in the distribution-free and finite samples setting. Still, we can analyze the asymptotic regime, similarly to Theorem 1 of \citet{sesia2020}, which proves the asymptotic conditional validity of CQR (without the presence of missing values) under consistency assumptions on the underlying quantile regressor. Here, by contrast, we study the asymptotic conditional validity of the impute-then-predict+conformalization procedure, by analyzing the consistency of impute-then-regress in Quantile Regression (QR). That is, we aim at showing that we satisfy the required assumption of consistency to invoke Theorem 1 of \citet{sesia2020}. The proofs of this section are given in \Cref{app:infinite_data}.

To analyze the consistency of impute-then-predict procedures for QR, we extend the work of \citet{lemorvan2021} on mean regression. QR with missing values, for a quantile level $\beta$, aims at solving
\begin{equation}
\label{eq:risk_minimization}
 \min_{f : \mathcal{X} \times \mathcal{M} \rightarrow \mathds{R}} \mathcal{R}_{\ell_\beta}(f) := \mathds{E}\left[ \ell_\beta \left(Y, f\left(X,M\right)\right) \right],
\end{equation}
with $\ell_\beta$ the pinball loss $\ell_\beta(y, \hat y) = \rho_\beta(y-\hat y)$ and  ${\rho_\beta(u) = \beta|u|\mathds{1}_{\{u \geq 0\}} + (1-\beta)|u|\mathds{1}_{\{u \leq 0\}}}$.

An associated $\ell_\beta$-Bayes predictor minimizes Eq. \eqref{eq:risk_minimization}. Its risk is called the $\ell_\beta$-Bayes risk, noted $\mathcal{R}_{\ell_\beta}^*$. 
Impute-then-predict procedure in QR aims at solving
\begin{equation}
\label{eq:risk_minimization_imp}
\min_{g : \mathcal{X} \rightarrow \mathds{R}} \mathcal{R}_{\ell_\beta,\Phi}(g) := \mathds{E}\left[ \ell_\beta \left(Y, g \circ \Phi\left( X,M \right) \right) \right],
\end{equation}
for $\Phi$ any imputation. Let $ g^*_{\ell_\beta, \Phi} \in \argmin_{g} \; \; \mathcal{R}_{\ell_\beta,\Phi}(g) $.
The following proposition states that $\mathcal{R}_{{\ell_\beta},\Phi}(g^*_{{\ell_\beta}, \Phi}) = \mathcal{R}_{\ell_\beta}^*$  and the consistency of a universal learner.
\begin{proposition}[$\ell_\beta$-consistency of an universal learner]
\label{prop:qr_bayes}
Let $\beta \in [0,1]$. If $X$ admits a density on $\mathds{R}^d$, then, for almost all imputation function $\Phi \in \mathcal{F}^{I}_{\infty}$,  (i) $g^*_{{\ell_\beta}, \Phi} \circ \Phi$ is  $\ell_\beta$-Bayes-optimal (ii) any universally consistent algorithm for QR trained on the data imputed by $\Phi$ is $\ell_\beta$-Bayes-consistent (i.e., asymptotically in the training set size).
\end{proposition}
Note that this QR case does not require $\mathds{E}\left[ \varepsilon | X_{\obs(M)}, M\right] = 0$, contrary to the quadratic loss  case~\citep{lemorvan2021}.

We conclude our asymptotic analysis of conditional coverage with \Cref{cor:asymp_cond}.
\begin{corollary}
\label{cor:asymp_cond}
For any missing mechanism, for almost all imputation function $\Phi \in \mathcal{F}^{I}_{\infty}$, if $F_{Y|(X_{\obs(M)},M)}$ is continuous, a universally consistent quantile regressor trained on the imputed data set yields asymptotic conditional coverage.
\end{corollary}
In words, the intervals obtained by taking Bayes predictors of levels $\alpha/2$ and $1-\alpha/2$ are exactly valid conditionally to both the mask $M$ and the observed variables $X_{\obs(M)}$, if $F_{Y|(X_{\obs(M)},M)}$ is continuous. 
Importantly, while this result is asymptotic, it holds for \textit{any} missing mechanism and it considers individualized conditional~coverage. 

\section{Empirical study}
\label{sec:experiments}
\textbf{Setup.} In all experiments, the data are imputed using iterative regression (\texttt{iterative ridge} implemented in Scikit-learn, \citet{scikit-learn}).\footnote{Theoretical results hold for any symmetric imputation.  In practice, constant, mean and MICE imputations gave similar results.} We compare the performance of our CQR-MDA-Exact and CQR-MDA-Nested (that is CP-MDA based on CQR) to CQR as well as to a vanilla QR (without any calibration). The predictive models are fitted on the imputed data concatenated with the mask. Without concatenating the mask to the features, the mask-conditional coverage of QR is worsened, as demonstrated in \Cref{sec:th_glmdata}. The prediction algorithm is a Neural Network (NN), fitted to minimize the pinball loss \citep[see \Cref{app:exp_set} for details]{chr}. For the vanilla QR, we use both the training and calibration sets for training.

\textbf{Synthetic and semi-synthetic experiments.} We designed the training and calibration data to have 20\% of MCAR values. To evaluate the test marginal coverage $\mathds{P} ( Y \in \widehat C_\alpha(X,M) )$, missing values are introduced in the test set according to the same distribution as on the training and calibration sets. Then, to compute an estimator of $\mathds{P} (Y \in \widehat C_\alpha(X,m) | M = m )$ for each $m \in \mathcal{M}$, we fix to a constant the number of observations per pattern,  to ensure that the variability in coverage is not impacted by $\mathds{P}\left( M = m \right)$. All experiments are repeated 100 times with different splits.

\subsection{Synthetic experiments: Gaussian linear data}

\textbf{Data generation.} The data is generated with $d=10$ according to \Cref{mod:glm}, with $X \sim \mathcal{N}\left(\mu, \Sigma \right)$, $\mu = (1,\cdots,1)^T$ and $\Sigma = \varphi (1,\cdots,1)^T(1,\cdots,1)+(1-\varphi)I_d$, $\varphi=0.8$, Gaussian noise $\varepsilon \sim \mathcal{N}(0,1)$ and the following regression coefficients $\beta = (1, 2, -1, 3, -0.5, -1, 0.3, 1.7, 0.4, -0.3)^T$\footnote{For dimension 3, in \Cref{fig:poc_toy}, the same model is used, keeping only the 3 first features and their associated parameters.}. Here, the oracle intervals are known (\Cref{prop:oracle_intervals}).

\textbf{Lowest and highest mask coverage, and associated length.} \Cref{fig:poc_masking,fig:poc_masking_worst} (\Cref{app:exp_synth}) and \Cref{fig:poc_masking_best} (\Cref{app:exp_synth}) show the lowest and highest mask coverage and their associated length as a function of the  training set size. The calibration size is fixed to 1000 and the test set contains 2000 points with the mask leading to the lowest coverage (here it corresponds to cases where only $X_4$ is observed) and 2000 points with the mask leading to the highest coverage (here it corresponds to all the variables observed). These figures highlight that:
\begin{itemize}[topsep=0pt,noitemsep,leftmargin=*]
	\item \textbf{CQR} and \textbf{QR} conditional coverage improve when the training size increases (\Cref{cor:asymp_cond});
	\item \textbf{Both versions of CQR-MDA} are MCV (\Cref{prop:meth_cond});
	\item \textbf{CQR-MDA-Exact} is exactly MCV as highest and lowest mask coverage are exactly 90\% (\Cref{prop:meth_cond});
    \item \textbf{CQR-MDA-Exact}'s lengths converge to the oracle~ones with increasing training size, showing it is not conservative, while \textbf{CQR-MDA-Nested} is overly conservative.
\end{itemize}

\begin{figure*}[!t]
    \centering
    \vspace{-8pt}
    \includegraphics[width=0.65\textwidth]{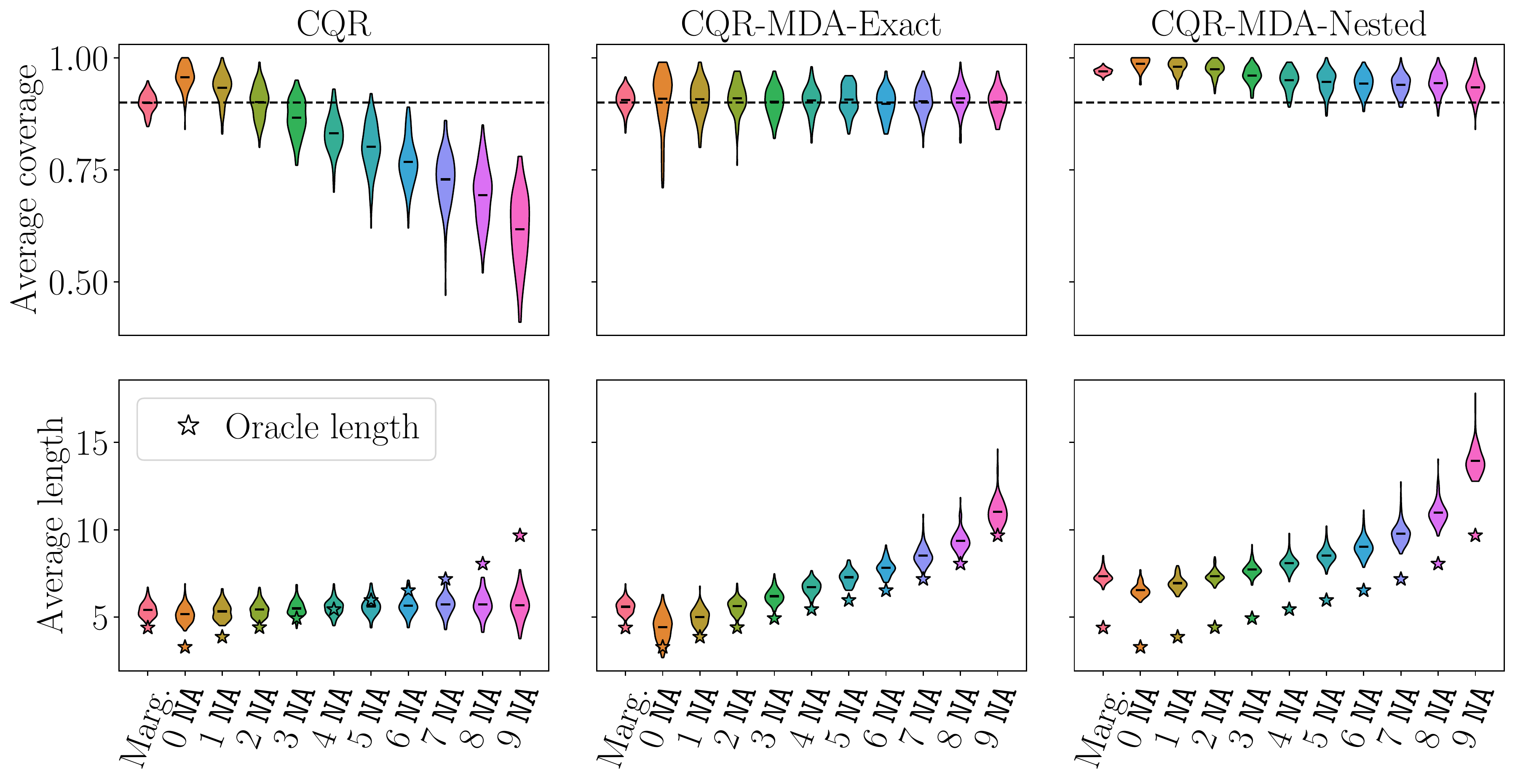}
    \vspace{-5pt}
    \caption{Average coverage (top) and length (bottom) as a function of the number of missing values (\texttt{NA}). The first violin plot shows the marginal coverage.
    $\#\Tr = 500$ and $\#\Cal = 250$. The marginal test set includes 2000 observations. The mask-conditional test set includes 100 individuals for each missing data pattern size.}
    \label{fig:d10_cov_len}
    \vspace{-7pt}
\end{figure*}

\textbf{Coverage and length by mask size.}
\Cref{fig:d10_cov_len} displays the average coverage and intervals' length as a function of the pattern size, i.e., the performance metrics are aggregated by the masks with the same number of missing variables; the first violin plot of each panel corresponds to the marginal coverage (see \Cref{app:exp_synth} for QR results). Note that only the pattern sizes are presented and not the patterns themselves as there are $2^d = 1024$ possible masks.\footnote{Note that in practice the  relationship between the coverage and the number of missing values is not necessarily monotonic as a mask with only one missing value can lead to more uncertainty than a mask with many missing values, see \Cref{app:glm}.} For each pattern size, 100 observations are drawn according to the distribution of $M | \text{size}(M)$ in the test set.
The training and calibration sizes are respectively 500 and 250 (\Cref{fig:d10_cov_len_1000} contains the results for other sizes). \Cref{fig:d10_cov_len} shows that:
\begin{itemize}[topsep=0pt,noitemsep,leftmargin=*]
	\item\textbf{CQR} is marginally valid (\Cref{prop:marg_cp_na});
	\item \textbf{CQR} and \textbf{QR} undercover with an increasing number of missing values. This can be explained because their length nearly does not vary with the size of the missing pattern, despite having the mask concatenated with the features;
	\item \textbf{Both versions of CQR-MDA} are marginally valid (Th.~\ref{prop:meth_marg}) and mask(-size)-conditionally-valid (Th.~\ref{prop:meth_cond});
	\item \textbf{CQR-MDA-Exact} is exactly mask(-size)-conditionally-valid (\Cref{prop:meth_cond}) and its length is close to the oracle ones. It has more variability for the patterns with few missing values as for these masks $\rm{Cal}^{(\text{test})}$ is smaller.
\end{itemize}

Similar experiments with 40\% of missing values are available in \Cref{app:exp_more_na}. Briefly, it corresponds to a setting where \mask is preferable over \masksub as the former outputs smaller intervals and is less variable.

\subsection{Semi-synthetic experiments}

\begin{figure*}[!t]
    \centering
    \includegraphics[width=0.98\textwidth]{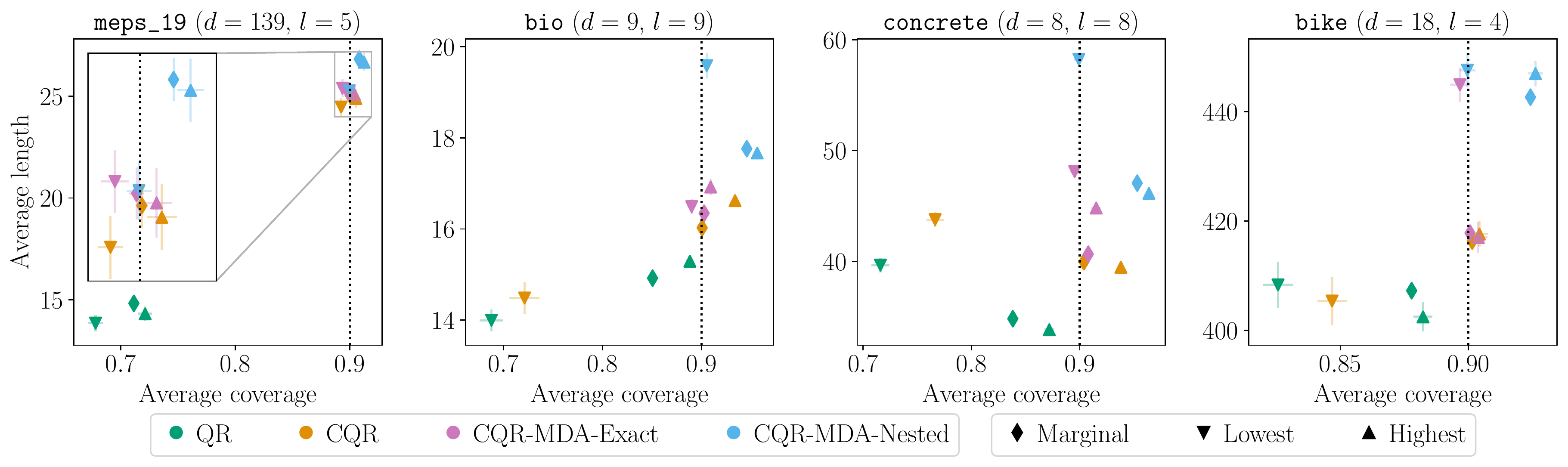}
    \vspace{-5pt}
    \caption{Validity and efficiency with missing values for 4 data sets (panels) with $d$ features, including $l$ quantitative ones in which missing values are introduced with probability 0.2. Colors represent the methods. Diamonds ($\blacklozenge$) represent marginal coverage while the patterns giving the lowest and highest mask coverage are represented with triangles ($\blacktriangledown$ and $\blacktriangle$). Vertical dotted lines represent the target coverage.}
    \label{fig:semi_synth}
\end{figure*}

We consider 6 benchmark real data sets for regression: \texttt{meps_19}, \texttt{meps_20}, \texttt{meps_21} \citep{MEPS}, \texttt{bio}, \texttt{bike} and \texttt{concrete} \citep{dua}, where we introduce missing values in their quantitative features, each of them having a probability 0.2 of being missing (i.e. it is a MCAR mechanism), as in the synthetic experiments.  Note that therefore some patterns have a  low (or null) frequency of appearance in the training sets of \texttt{bio} and \texttt{concrete}. The sample sizes for training, calibration, and testing, and simulation details are provided in \Cref{app:exp_semi_synth}, along with results for smaller training and calibration sets.
 
\Cref{fig:semi_synth} depicts the results by combining \textit{validity} and \textit{efficiency} (length) for \texttt{meps_19}, \texttt{bio}, \texttt{concrete}, and \texttt{bike}, where this graph follows the visualization used in \citet{pmlr-v162-zaffran22a}.
The results for \texttt{meps_20} and \texttt{meps_21} are given in \Cref{app:exp_semi_synth}, as they are similar to \texttt{meps_19}.

Each of the panels in \Cref{fig:semi_synth} summarizes the results for one data set, with the average coverage shown in the $x$-axis and the average length in the $y$-axis. A method is mask-conditionally-valid if all the markers of its color are at the right of the vertical dotted line (90\%). The design of \Cref{fig:semi_synth} requires a different interpretation than \Cref{fig:d10_cov_len} (or the subsequent \Cref{fig:real_data}). For each method we report, for the pattern having the highest (or lowest) coverage, its length and coverage. However, as  this pattern may depend on the method, the length for the highest/lowest should not be directly compared between methods.
We observe that:
\begin{itemize}[topsep=0pt,noitemsep,leftmargin=*,wide]
	\item \textbf{CQR} is marginally valid (\textcolor{blindorange}{orange $\blacklozenge$}, \Cref{prop:marg_cp_na}), but not MCV as the lowest mask coverage (\textcolor{blindorange}{orange $\blacktriangledown$}) is far below 90\% (\texttt{bio}, \texttt{concrete}, and \texttt{bike} data sets);
    \item \textbf{CQR-MDA-Exact} is marginally valid (\textcolor{blindpurple}{purple $\blacklozenge$}, \Cref{prop:meth_marg}). It is also exactly MCV, as the lowest (\textcolor{blindpurple}{purple $\blacktriangledown$}) and highest (\textcolor{blindpurple}{purple $\blacktriangle$}) mask coverages are about 90\% (\Cref{prop:meth_cond});
    \item \textbf{CQR-MDA-Nested} is marginally valid (\textcolor{blindblue}{blue $\blacklozenge$}, \Cref{prop:meth_marg}). It is also MCV, as the lowest (\textcolor{blindblue}{blue $\blacktriangledown$}) mask coverage is larger than 90\% (\Cref{prop:meth_cond}).
\end{itemize}

\subsection{Predicting the level of platelets for trauma patients}

\begin{figure}[!b]
\vspace{-12pt}
\begin{center}
\includegraphics[width=0.4\textwidth]{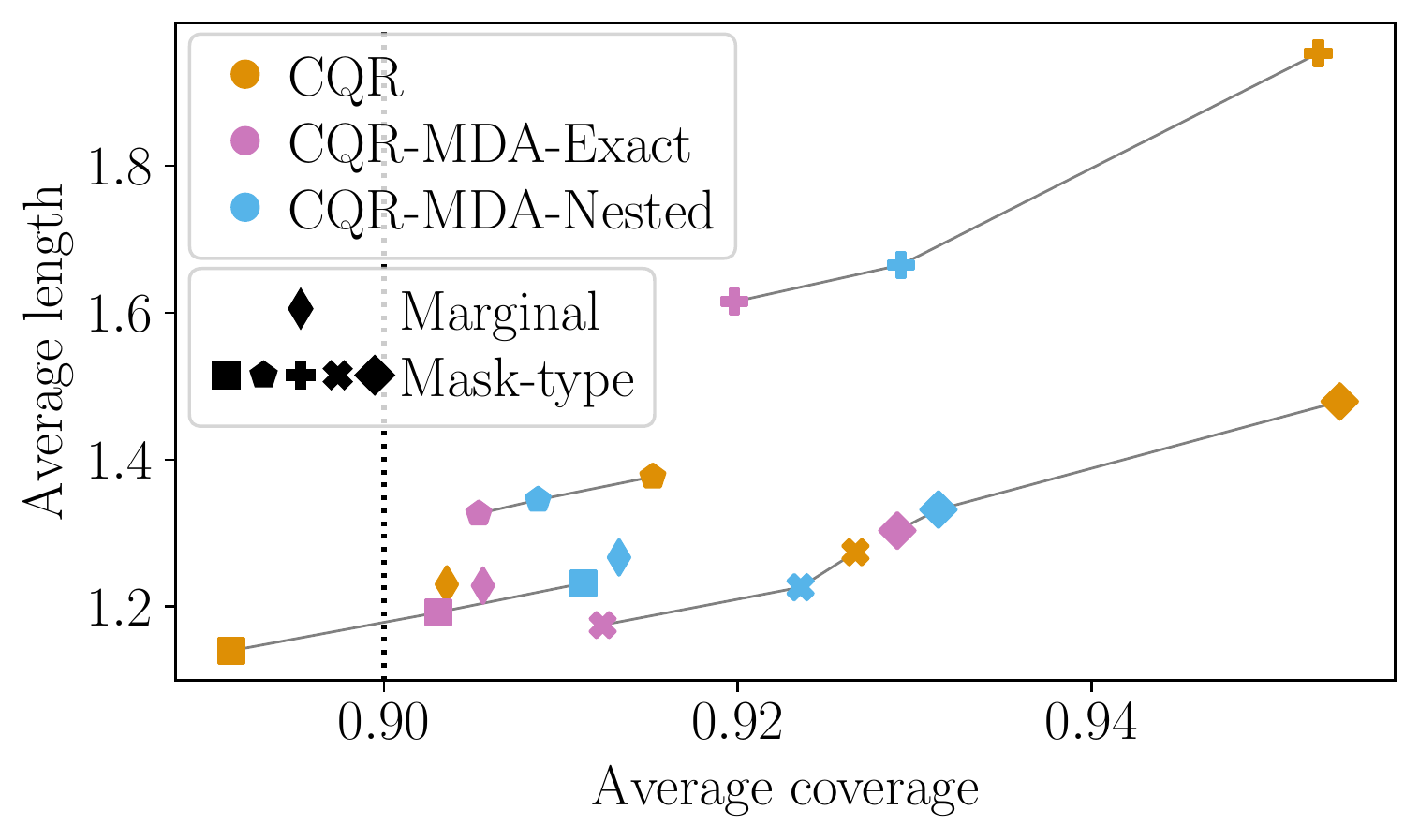}
\end{center}
\vspace{-14pt}
\caption{Average coverage and length on the TraumaBase® analysis. See the caption of \Cref{fig:semi_synth} for details. Other symbols than diamond correspond to computing the average per mask. Each individual's prediction is obtained by using 15390 observations for training, and 7694 for calibration.}
\label{fig:real_data}
\end{figure}

We study the applicability and robustness of CPMDA on the critical care TraumaBase® data. We focus on predicting the level of platelets of severely injured patients upon arrival at the hospital. This level is directly related to the occurrence of hemorrhagic shock and is difficult to obtain in real-time: predicting it accurately could be crucial to anticipate the need for transfusion and blood resources. In addition, this prediction task appears to be challenging as \citet{wei} achieved an average relative prediction error (${\Vert \hat y - y \Vert^2}/{\Vert y \Vert^2}$) that is no lower than 0.23. This highlights the need for reliable uncertainty quantification.

After applying inclusion and exclusion criteria obtained by medical doctors and following the pipeline of \citet{sportisse} described in \Cref{app:exp_real}, we left with a subset of 28855 patients and 7 features. Missing values vary from 0\% to 24\% by features, with a total average of~7\%.

\textbf{Results.} The results are summarized in \Cref{fig:real_data}, where we use different markers to denote the different masks. To ensure a fair comparison between the conformal methods, we only keep the missing patterns for which there are more than 200 individuals; this excludes 7 patterns.  
Finally, since we found that the vanilla QR tends to be overly conservative, we refer to \Cref{app:exp_real} for its results. \Cref{fig:real_data} shows that all conformal approaches achieve marginal coverage higher than the desired 90\% level (diamonds $\blacklozenge$). Furthermore, for each mask (each set of linked markers) \textbf{CQR-MDA} improves coverage compared to \textbf{CQR} by approaching 90\%, and efficiency by reducing the average length. Noticeably, for the pattern corresponding to all features observed (squares $\blacksquare$), \textbf{CQR-MDA} has a coverage rate above 90\% while \textbf{CQR} is below the target level. Therefore, we believe \textbf{CQR-MDA} should be recommended as it improves upon the vanilla impute-then-regress+CQR approach.

\section{Conclusion and perspectives}
In this paper, we study the interplay between uncertainty quantification and missing values. We show that missing values introduce heteroskedasticity in the prediction task. This brings challenges on how to provide uncertainty estimators that are valid conditionally on the missing patterns, which are addressed by this work.
Our analysis leaves several directions open: (1) obtaining results \textit{beyond the MCAR assumption} for CP-MDA, both theoretically and numerically, (2) extending the (numerical) analysis to non-split approaches, (3) investigating the numerical performances of other conditional CP approaches (such as \citet{sesia2020,izbicki,izbicki_jmlr,lin}), (4) studying the impact of the imputation on QR with finite samples. A more detailed discussion on these directions is provided in \Cref{app:perspectives}.

\section*{Acknowledgements}
We thank Baptiste Goujaud for fruitful discussions. We sincerely thank anonymous reviewers for their feedbacks which improved the paper. This work was supported by a public grant as part of the Investissement d'avenir project, reference ANR-11-LABX-0056-LMH, LabEx LMH. M. Zaffran has been awarded the 2022 Scholarship for Mathematics granted by the Séphora Berrebi Foundation which she gratefully thanks for its support. The work of A. Dieuleveut is partially supported by ANR-19-CHIA-0002-01/chaire SCAI and Hi!~Paris. The work of J. Josse is partially supported by ANR-16-IDEX-0006. Y. Romano was supported by the ISRAEL SCIENCE FOUNDATION (grant No. 729/21). He also thanks the Career Advancement Fellowship, Technion, for providing additional research support.

\newpage

\bibliography{cp_na_cleaned}
\bibliographystyle{apalike}

\newpage

\appendix

\onecolumn

\part*{Appendices}

The appendices are organized as follows. 

\Cref{app:perspectives} provides a more detailed discussion  on open directions and perspectives.

\Cref{app:cqr} describes CQR, used in the paper.

\Cref{app:itp+conf} provides an explicit description of impute-then-predict+conformalization (\Cref{app:itp+conf_alg}), along with its proof of validity, that is the proofs for \Cref{sec:cp_na_marg} (\Cref{app:itp+conf_proof}). 

Then, \Cref{app:glm} contains the proofs for the Gaussian linear model oracle intervals presented in \Cref{sec:th_glmdata} (\Cref{app:glm_oracle}), along with the discussion on how mean-based approaches fail (\Cref{app:glm_discussion}).

\Cref{app:methods} gives the general statement of \masksub (\Cref{app:methods_general}), and the proofs of the validity theorems for \masksub (\Cref{app:methods_exact}), along with the theoretical study of \mask (\Cref{app:methods_nested}).

\Cref{app:infinite_data} provides all the proofs about consistency and asymptotic conditional coverage presented in \Cref{sec:infinite}.

Finally, \Cref{app:exp} contains all the details for the experimental study and additional results completing \Cref{sec:experiments}. More precisely, \Cref{app:exp_set} gives more details about the settings. \Cref{app:exp_synth} contains results on synthetic data \camerareadyrevisionlast{with 20\% of MCAR missing values, while \Cref{app:exp_more_na} shows the results on synthetic data when the proportion of MCAR missing values is 40\%}. \Cref{app:exp_semi_synth} describes the real data sets used for the semi-synthetic experiments, and presents the remaining results. \Cref{app:exp_real} presents the real medical data set (TraumaBase®), the pipeline and settings used and the results obtained by QR on this data set.

\camerareadyrevision{\section{Detailed perspective discussion}
\label{app:perspectives}
First, obtaining results \textit{beyond the MCAR assumption} for CP-MDA. On the numerical side, preliminary experiments show promising results, indicating CP-MDA's robustness, but a detailed numerical study is needed. On the theoretical side, understanding the  limits of CP-MDA validity is of high importance. Results without assumptions on the missingness distribution seem impossible to obtain. Even with MAR data, the task of pointwise prediction can be very challenging if the output distribution strongly depends on the pattern \citep{ayme2022}. As the impossibility results of conditional validity \citep{lei_distribution-free_2014,vovk_conditional_2012,barber_limits_2021}, assumptions on the missing mechanism are needed.

Second, extending the (numerical) analysis to non-split approaches (e.g., based on the Jackknife) would be relevant, as it could improve the base model and therefore how the heteroskedasticity is taken into account. Note that CP-MDA can be written to take into account this splitting strategy, and thus our theoretical results on MCV would directly extend.

Third, investigating the numerical performances of other conditional CP approaches (such as \citet{sesia2020,izbicki,izbicki_jmlr,lin}) within the MDA framework is of interest. In this paper, we analyze empirically the instance of CP-MDA on top of CQR as it is the simplest version of QR based CP, but the theory and motivation of this work is not specific to CQR. Exactly as CQR, none of the aforementioned methods would provide MCV if used out of the box. But if combined with CP-MDA, then all of them will be granted MCV.

Finally, while our approach is to be agnostic to the imputation chosen (similarly to CP being agnostic to the underlying model), an interesting research path is to study the impact of the imputation on QR with finite samples. }

\section{Illustration and details on CQR \citep{romano_conformalized_2019} procedure}
\label{app:cqr}

\Cref{fig:cqr_scheme} provides a visualization and step by step description of  CQR. 

\begin{figure}

    \begin{minipage}{0.49\textwidth}
        \includegraphics[scale=0.45]{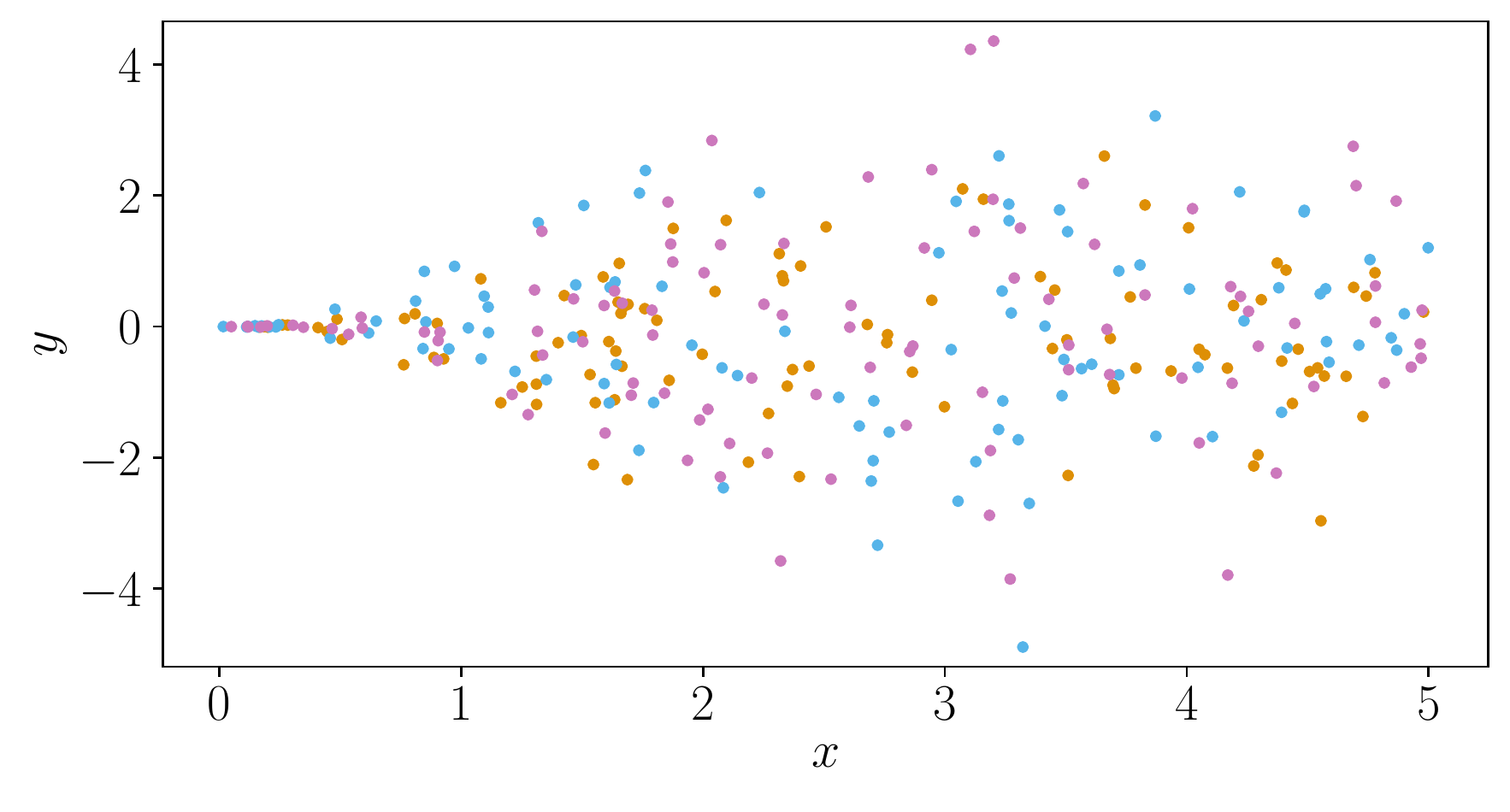}
    \end{minipage}
    \hfill
    \begin{minipage}{0.5\textwidth}
        \begin{itemize}
            \item[$\blacktriangleright$] Create a \textcolor{blindorange}{proper training set}, a \textcolor{blindblue}{calibration set}, and keep your \textcolor{blindpurple}{test set}, by randomly splitting your data set.
        \end{itemize}
    \end{minipage}
    
    \begin{minipage}{0.06\textwidth}

    \textbf{Step 1}
    \end{minipage}
    \hfill
    \begin{minipage}{0.43\textwidth}
         \includegraphics[scale=0.5]{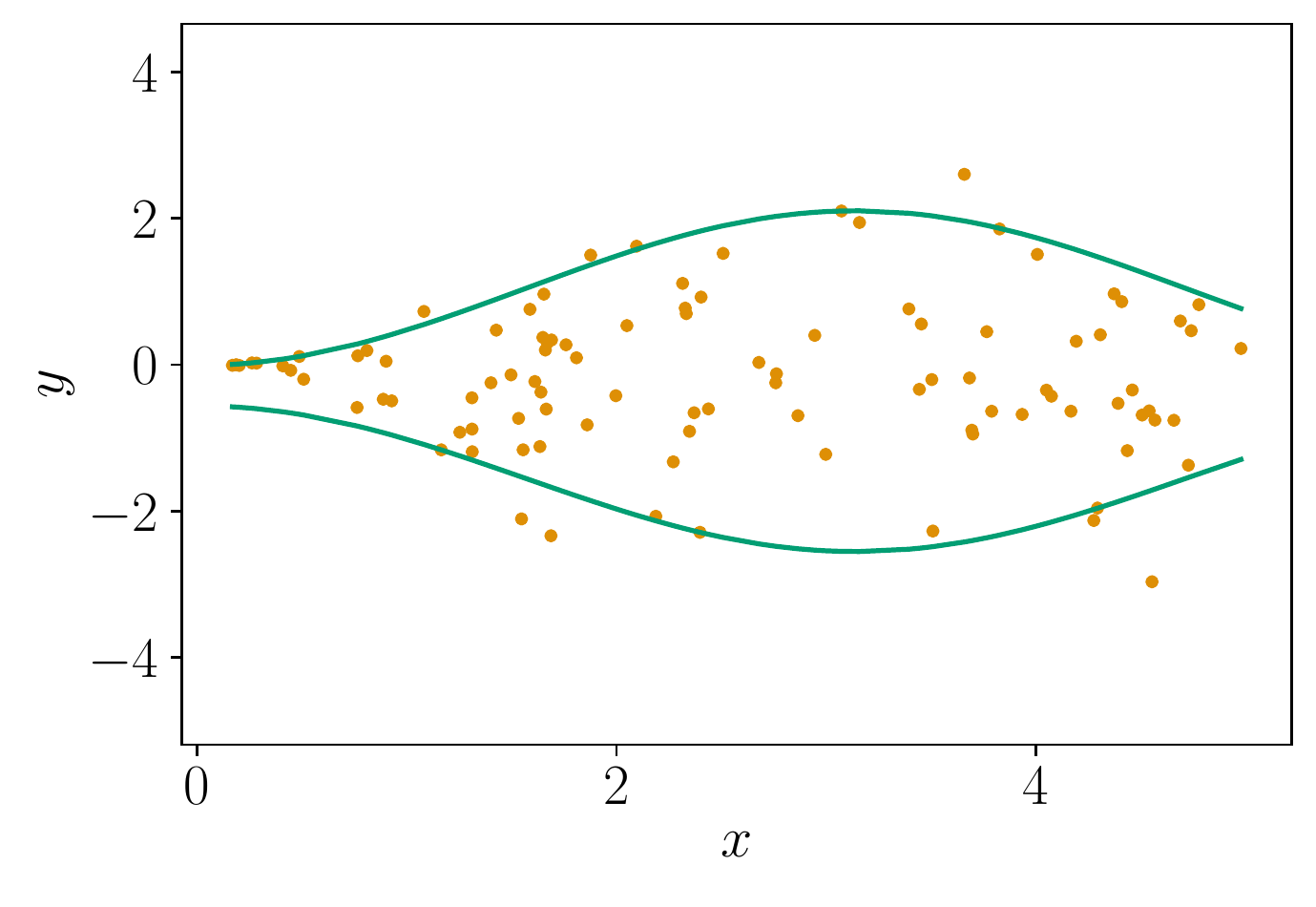}
    \end{minipage}
    \hfill
    \begin{minipage}{0.5\textwidth}
        On the \textcolor{blindorange}{proper training set}:
        \begin{itemize}
        \item[$\blacktriangleright$] Learn \textcolor{blindgreen}{$\hat{q}_{\text{low}}$} and \textcolor{blindgreen}{$\hat{q}_{\text{upp}}$} 
    \end{itemize}
    \end{minipage}

    \begin{minipage}{0.06\textwidth}
    \textbf{Step 2}
    \end{minipage}
    \hfill
    \begin{minipage}{0.43\textwidth}
         \includegraphics[scale=0.335]{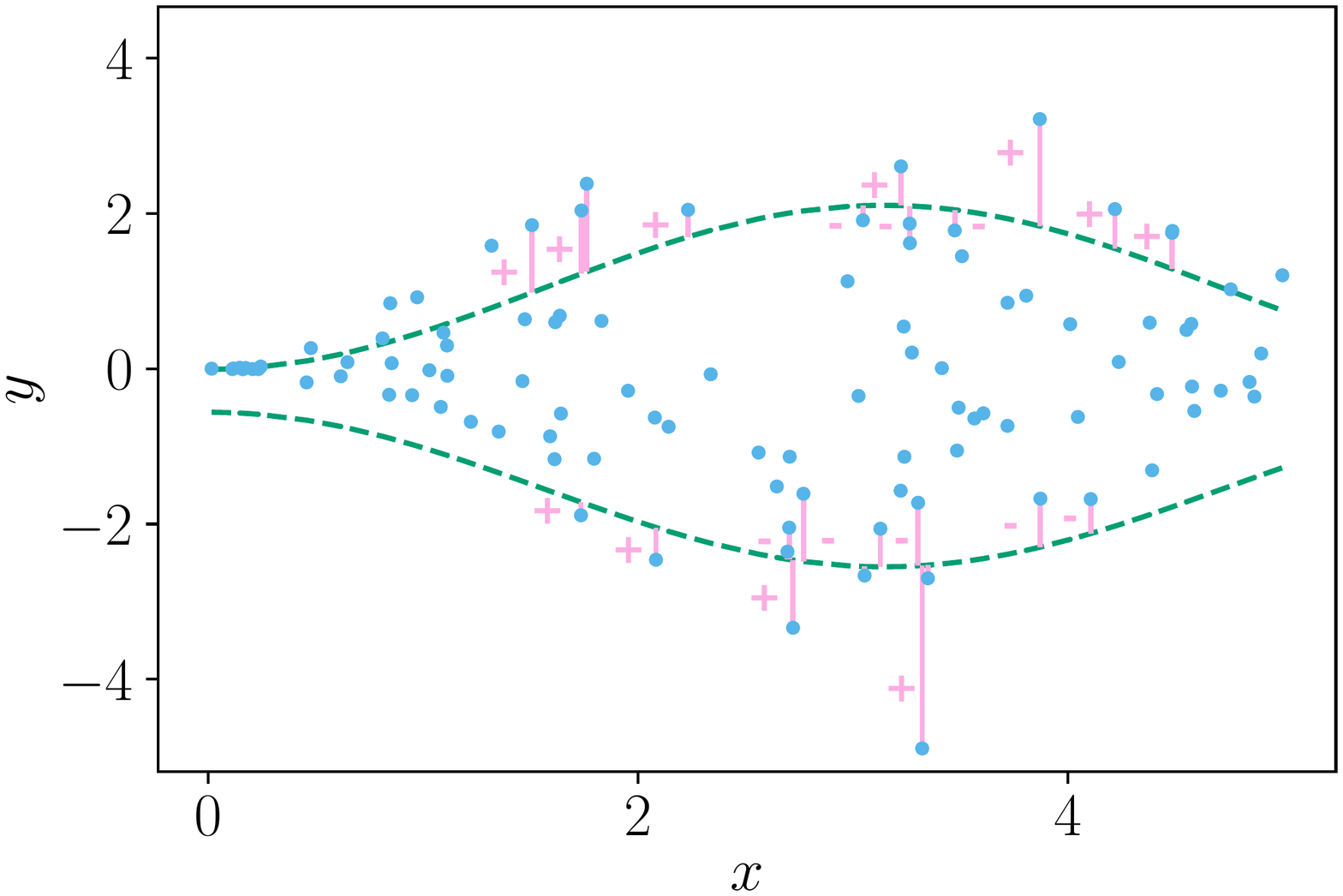}
    \end{minipage}
    \hfill
    \begin{minipage}{0.5\textwidth}
        On the \textcolor{blindblue}{calibration set}:
        \begin{itemize}
        \item[$\blacktriangleright$] Predict with \textcolor{blindgreen}{$\hat{q}_{\text{low}}$} and \textcolor{blindgreen}{$\hat{q}_{\text{upp}}$} 
        \item[$\blacktriangleright$] Get the scores \\ ${\textcolor{blindpink}{s^{(k)}} = \max\left\{\textcolor{blindgreen}{\hat{q}_{\text{low}}}\left(x^{(k)}\right) - y^{(k)}, y^{(k)} - \textcolor{blindgreen}{\hat{q}_{\text{upp}}}\left(x^{(k)}\right) \right\}}$
        \item[$\blacktriangleright$] Compute the $(1-\alpha)\times(1+\frac{1}{\#\rm{Cal}})$ empirical quantile of the \textcolor{blindpink}{$s^{(k)}$}, noted \textcolor{blindpink}{$\widehat Q_{1-\hat\alpha}\left(S\right)$}
    \end{itemize}
    \end{minipage}
 
    \begin{minipage}{0.06\textwidth}
    \textbf{Step 3}
    \end{minipage}
    \hfill   
    \begin{minipage}{0.43\textwidth}
         \includegraphics[scale=0.5]{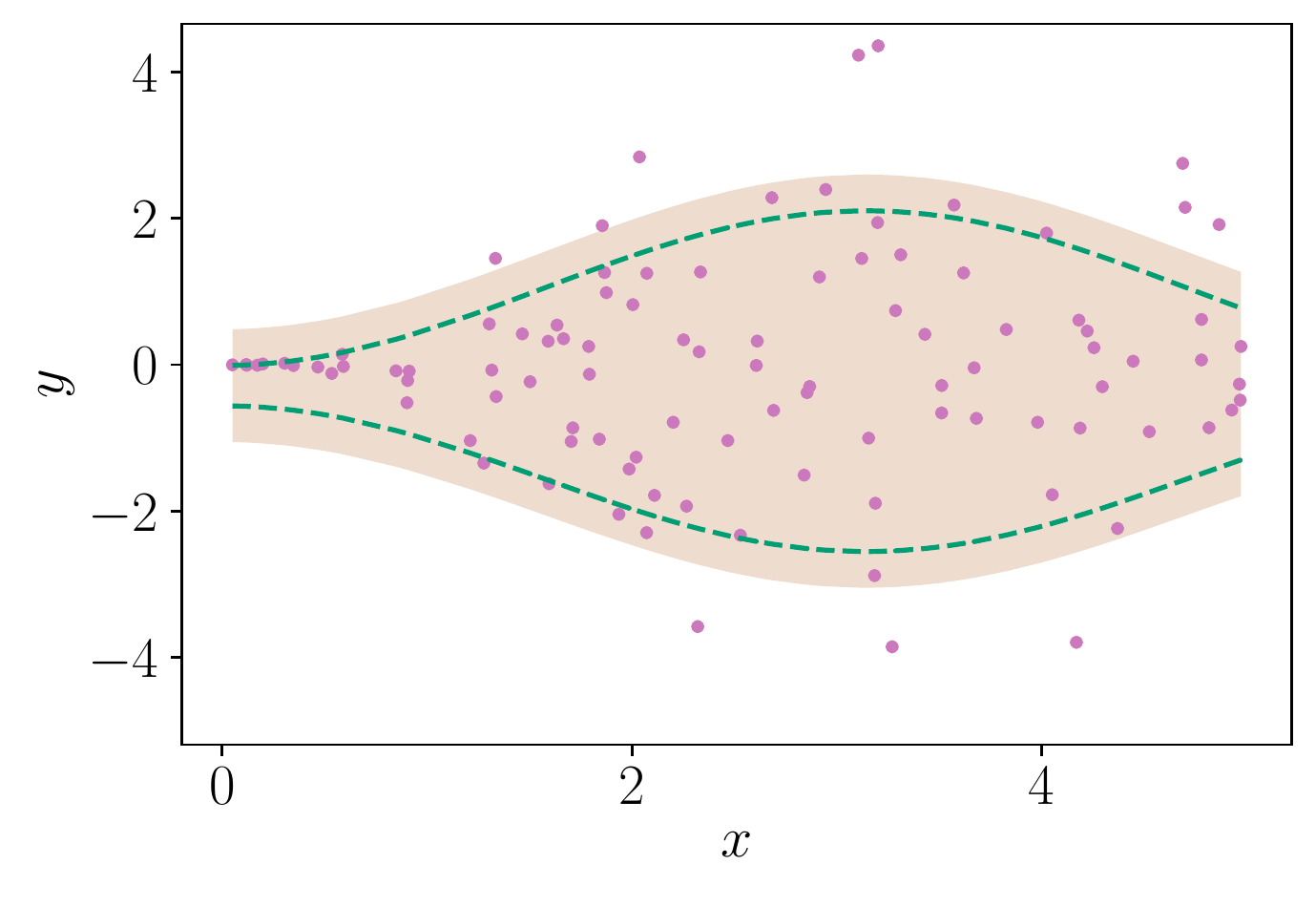}
    \end{minipage}
    \hfill
    \begin{minipage}{0.5\textwidth}
        On the \textcolor{blindpurple}{test set}:
        \begin{itemize}
        \item[$\blacktriangleright$] Predict with \textcolor{blindgreen}{$\hat{q}_{\text{low}}$} and \textcolor{blindgreen}{$\hat{q}_{\text{upp}}$} 
        \item[$\blacktriangleright$] Build $\hat{C}_{\hat{\alpha}}(x)$: \textcolor{blindbrown}{$[  \hat{q}_{\text{low}}(x) - \widehat Q_{1-\hat\alpha}\left(S\right), \hat{q}_{\text{upp}}(x) + \widehat Q_{1-\hat\alpha}\left(S\right)]$}
    \end{itemize}
    \end{minipage}
    
    \caption{Schematic illustration of Conformalized Quantile Regression (CQR) \citep{romano_conformalized_2019}.}
    \label{fig:cqr_scheme}

\end{figure}

\section{Impute-then-predict+conformalization}
\label{app:itp+conf}

\subsection{Description of the algorithm}
\label{app:itp+conf_alg}
\begin{algorithm}
\caption{SCP on impute-then-predict}
\label{alg:cp_na}
\begin{algorithmic}[1] 
\REQUIRE Imputation algorithm $\mathcal{I}$, predictive algorithm $\mathcal{A}$, conformity score function $s$, significance level $\alpha$, training set $\left\{ \left(X^{(1)},M^{(1)},Y^{(1)}\right), \cdots, \left(X^{(n)},M^{(n)},Y^{(n)} \right) \right\}$.
\ENSURE Prediction interval $\widehat{{C}}_{\alpha}\left( X,M \right)$. 
\STATE Randomly split $\{1, \ldots, n\}$ into two disjoint sets $\rm{Tr}$ and $\rm{Cal}$. 
\STATE Fit the imputation function: $\Phi(\cdot) \leftarrow \mathcal{I}\left(\left\{ \left( X^{(k)}, M^{(k)} \right), k \in \rm{Tr}\right\}\right)$ 
\STATE Impute the data set: $\left\{ X_{\text{imp}}^{(k)} \right\}_{k=1}^n := \left\{  \Phi\left(X^{(k)},M^{(k)}\right) \right\}_{k=1}^n$
\STATE Fit algorithm $\mathcal{A}$: $\hat{g}(\cdot) \leftarrow \mathcal{A}\left(\left\{\left( X_{\text{imp}}^{(k)}, Y^{(k)} \right), k \in \rm{Tr}\right\}\right)$ 
\FOR {$k \in \rm{Cal} $} 
\STATE Set $S^{(k)} = s \left( Y^{(k)}, \hat{g}\left(X_{\text{imp}}^{(k)}\right) \right)$, the \textit{conformity scores}
\ENDFOR
\STATE Set $\mathcal{S}_{\rm{Cal}} = \{S^{(k)}, k \in \rm{Cal}\}$
\STATE Compute $\widehat Q_{1-\alpha^{\rm{SCP}}}\left(\mathcal{S}_{\rm{Cal}}\right)$, the $1-\alpha^{\rm{SCP}}$-th empirical quantile of $\mathcal{S}_{\rm{Cal}}$, with $1-\alpha^{\rm{SCP}} := (1-\alpha)\left(1+1/\#\rm{Cal}\right)$.
\STATE Set  $\widehat{{C}}_{\alpha}\left( X,M \right) = \left\{y \text{ such that } s\left( y, \hat{g} \circ \Phi \left( X, M \right) \right) \leq \widehat{Q}_{1-\alpha^{\rm{SCP}}}\left(\mathcal{S}_{\rm{Cal}}\right)  \right\}$.
\end{algorithmic}
\end{algorithm}

Similarly, \Cref{alg:cp_na_sub} can be written to include any underlying predictive algorithm (regression or 
classification) and any score function.

\newpage
\subsection{Proof of exchangeability after imputation}
\label{app:itp+conf_proof}

In this subsection, we provide a more formal  statement of \Cref{lem:exch_imp} and \Cref{prop:marg_cp_na} in respectively \Cref{lem:exch_imp_T} and \Cref{prop:marg_cp_na_T}. To that end, we introduce a notion of symmetrical imputation \textit{on a set $\mathcal T$}, for $\mathcal T \subset \llbracket 1, n+1\rrbracket$.

\begin{assumption}[Symmetrical imputation on a set $\mathcal T$]
\label{ass:imp_sym_T}
For a given set of points $\{X^{(k)}, M^{(k)}, Y^{(k)}\}_{k\in \mathcal T} $ the imputation function $\Phi$ is the output of an algorithm $\mathcal{I}$ that treats the  data points in $\mathcal T$ symmetrically: $\mathcal{I}(\{X^{(k)}, M^{(k)}, Y^{(k)}\}_{k\in \mathcal T}) \overset{(d)}{=} \mathcal{I}(\{X^{(\sigma(k))}, M^{(\sigma(k))}, Y^{(\sigma(k))}\})_{k\in \mathcal T} $ 
conditionally to $\{X^{(k)}, M^{(k)}, Y^{(k)}\}_{k\in \mathcal T} $ and for any permutation $\sigma$ on~$\llbracket 1,\# \mathcal T \rrbracket$. 
\end{assumption}

\begin{lemma}[Imputation preserves exchangeability]
\label{lem:exch_imp_T}

Let \ref{ass:iid} hold. Then, for any missing mechanism, for any imputation function $\Phi$ satisfying \ref{ass:imp_sym_T}, the imputed random variables $\left(\Phi\left(X^{(k)},M^{(k)} \right),M^{(k)},Y^{(k)}\right)_{k\in \mathcal T}$ are exchangeable. 

\end{lemma}

\begin{proposition}[(Exact) validity of impute-then-predict+conformalization]
\label{prop:marg_cp_na_T}
If \ref{ass:iid} is  satisfied, then we have the following three results.
\begin{enumerate}
    \item \textbf{Full CP:} if \ref{ass:imp_sym_T} is satisfied for $\mathcal T = \llbracket 1, n+1 \rrbracket$ (i.e., the imputation algorithm treats \textit{all} points symmetrically), then 
impute-then-predict+Full CP is marginally valid. If moreover the scores are almost surely distinct, it is exactly valid.

OR

\item \textbf{Jackknife+} if \ref{ass:imp_sym_T} is satisfied for $\mathcal T = \llbracket 1, n+1 \rrbracket$ (i.e., the imputation algorithm treats \textit{all} points symmetrically), then 
impute-then-predict+Jackknife+ is marginally valid (of level $1-2\alpha$). 

    OR
    
    \item \textbf{SCP} with the split  $ \llbracket 1, n+1 \rrbracket = \rm{Tr} \bigcup \rm{Cal} \bigcup \rm{Test} $ and if \ref{ass:imp_sym_T} is satisfied for $\mathcal T = \rm{Cal} \bigcup \rm{Test}$ (i.e., the imputation treats all points in $\rm{Cal}\bigcup \rm{Test}$  symmetrically) then 
    impute-then-predict+conformalization is marginally valid. If moreover the scores are almost surely distinct, it is exactly valid.

\end{enumerate}
\end{proposition}

\begin{remark}[Imputation choices for SCP]
In the latter case, for SCP, the coverage result can be derived conditionally on $\rm{Tr}$, thus the coverage results holds for:
(i) any deterministic imputation function (conditionally on $\Tr$) (that is any arbitrary function of $\Tr$), or (ii) any stochastic imputation function treating $\Cal$ and $\Test$ symmetrically (iii) any combination of both.
\end{remark}

\begin{proof}[Proof of \Cref{lem:exch_imp_T}]

$\Phi$ is the output of an imputing algorithm $\mathcal{I}$ trained on $\left\{ \left(X^{(k)},M^{(k)},Y^{(k)} \right)_{k \in \mathcal{T}} \right\}$.

Assume $\left(X^{(k)},M^{(k)},Y^{(k)} \right)_{k \in \mathcal{T}}$ are exchangeable (\ref{ass:iid}).

Thus, if $\mathcal{I}$ treats the data points in $\mathcal{T}$ symmetrically, $\left( \Phi(X^{(k)},M^{(k)}),M^{(k)},Y^{(k)} \right)_{k \in \mathcal{T}}$ are exchangeable (see proof of Theorem 1b in \citep{bacarati_nexp} for example).

\end{proof}

\begin{proof}[Proof of \Cref{prop:marg_cp_na_T}]
\Cref{prop:marg_cp_na_T} is a consequence of \Cref{lem:exch_imp_T} with different choices of $\mathcal T$, that enable to apply the following results:
\begin{enumerate}[noitemsep,topsep=0pt]
    \item Full CP: \citet{vovk_algorithmic_2005}, also re-stated in \citet{bacarati_nexp}
    \item Jackknife+: \citet{barber2021jackknife}
    \item SCP: \citet{lei_distribution-free_2018} or \citet{papadopoulos_inductive_2002} and \citet{angelopoulos-gentle} for a generic version with any score function (note that the coverage is proved conditionally on $\Tr$).
\end{enumerate}
\end{proof}

\section{Gaussian linear model}
\label{app:glm}

\subsection{Distribution of $Y | (X_{\obs(m)}, M)$ and oracle intervals}
\label{app:glm_oracle}

\begin{proposition}[Distribution of $Y | ( X_{\obs(M)},M )$ \citep{lemorvan2020}]

Under \Cref{mod:glm}, for any $m \in \{0,1\}^d$:

\begin{equation*}
Y | ( X_{\obs(m)},M = m ) \sim \mathcal{N}\left( \tilde\mu^m, \widetilde\Sigma^m \right),
\end{equation*}
with:
\begin{align*}
\tilde\mu^m = & \; \beta^T_{\obs(m)} X_{\obs(m)} + \beta^T_{\mis(m)} \mu^m_{{\mis}|{\obs}} \\
\mu^m_{{\mis}|{\obs}} = & \; \mu^m_{\mis(m)} +  \Sigma^m_{\mis(m),\obs(m)} ({\Sigma^m_{\obs(m),\obs(m)}})^{-1}(X_{\obs(m)} - \mu^m_{\obs(m)}), \\
\widetilde\Sigma^m = & \; \beta^T_{\mis(m)} \Sigma^m_{{\mis}|{\obs}} \beta_{\mis(m)} + \sigma^2_{\varepsilon} \\
\Sigma^m_{{\mis}|{\obs}}= & \; \Sigma^m_{\mis(m),\mis(m)} - \Sigma^m_{\mis(m),\obs(m)} ({\Sigma^m_{\obs(m),\obs(m)}})^{-1}\Sigma^m_{\obs(m),\mis(m)}.
\end{align*}

\end{proposition}

\begin{proposition}[Oracle intervals]

Under \Cref{mod:glm}, for any $m \in \{0,1\}^d$, for any $\delta \in (0,1)$:
\begin{align*}
    q^{Y | (X_{\obs(m)}, M = m)}_\delta & = \beta^T_{\obs(m)} X_{\obs(m)} + \beta^T_{\mis(m)} \mu^{m}_{\mis|\obs} + q^{\mathcal{N}(0,1)}_\delta\sqrt{\beta^T_{\mis(m)} \Sigma^{m}_{\mis|\obs} \beta_{\mis(m)} + \sigma^2_{\varepsilon}},
\end{align*}
and the oracle predictive interval length is given by:
\begin{equation}
\mathcal{L}^*_{\alpha}(m) = 2 q^{\mathcal{N}(0,1)}_{1-\frac{\alpha}{2}} \sqrt{\beta^T_{\mis(m)} \Sigma^{m}_{\mis|\obs} \beta_{\mis(m)} + \sigma^2_{\varepsilon}}.
\end{equation}

\end{proposition}

\begin{proof}

Using multivariate Gaussian conditioning \citep{cond_gaussian}, for any subset of indices $L \in \llbracket 1, d\rrbracket$:
\begin{equation}
    X_K | (X_L, M) \sim \mathcal{N}(\mu^M_{K|L}, \Sigma^M_{K|L}), \label{eq:cond_gaussian_dist}
\end{equation}
with $K = \bar{L}$ (the complement indices) and:
\begin{align*}
    \mu^M_{K|L} & = \mu^M_K + \Sigma^M_{K,L} {\Sigma^M_{L,L}}^{-1}(X_L - \mu^M_L), \\
    \Sigma^M_{K|L} & = \Sigma^M_{K,K} -\Sigma^M_{K,L} {\Sigma^M_{L,L}}^{-1}\Sigma^M_{L,K}.
\end{align*}

Given that $Y = \beta^T X + \varepsilon$, with $\varepsilon \sim \mathcal{N}(0, \sigma^2_{\varepsilon}) \perp\!\!\!\!\perp (X, M)$, the following holds:
\begin{align*}
    Y | (X_L, M) & \overset{(d)}{=} (\beta^T X + \varepsilon) | (X_L, M) 
    \overset{(d)}{=} \beta^T_L X_L + (\varepsilon + \beta^T_{K} X_{K}) | (X_L, M)
\end{align*}
and by \Cref{eq:cond_gaussian_dist},  $\beta^T_{K} X_{K} | (X_L, M) \sim \mathcal{N}(\beta^T_K \mu^M_{K|L}, \beta^T_K \Sigma^M_{K|L} \beta_K)$, and $(\varepsilon |(X_L, M)) \sim \mathcal{N}(0, \sigma^2_{\varepsilon}) $, and $(\beta^T_{K} X_{K} \perp\!\!\!\!\perp \varepsilon) | (X_L, M)$ .
Thus:
\begin{equation*}
    Y | (X_L, M) \sim \mathcal{N}(\beta^T_L X_L + \beta^T_K \mu^M_{K|L}, \beta^T_K \Sigma^M_{K|L} \beta_K +  \sigma^2_{\varepsilon}).
\end{equation*}

Consequently, for any $\delta \in (0,1)$:
\begin{equation}
    q^{Y | (X_L, M)}_\delta = \beta^T_L X_L + \beta^T_K \mu^M_{K|L} + q^{\mathcal{N}(0,1)}_\delta\sqrt{\beta^T_K \Sigma^M_{K|L} \beta_K +  \sigma^2_{\varepsilon}}. \label{eq:quantiles_gaussian_KL}
\end{equation}

For any pattern $m \in \{0,1\}^d$, applying \Cref{eq:quantiles_gaussian_KL} with   $K= \mis(m)= \overline{\obs(m)}$, $L=\obs(m)$, we have, for any $\delta \in (0,1)$:
\begin{align*}
    q^{Y | (X_{\obs(m)}, M = m)}_\delta = & \beta^T_{\obs(m)} X_{\obs(m)} + \beta^T_{\mis(m)} \mu^m_{\mis|{\obs}} + q^{\mathcal{N}(0,1)}_\delta\sqrt{\beta^T_{\mis(m)} \Sigma^m_{{\mis}|{\obs}} \beta_{\mis(m)} +  \sigma^2_{\varepsilon}},
\end{align*}
and:
\begin{equation*}
    \mathcal{L}^*_{\alpha}(m) = 2 \times q^{\mathcal{N}(0,1)}_{1-\alpha/2} \times \sqrt{\beta^T_{\mis(m)} \Sigma^m_{{\mis}|{\obs}} \beta_{\mis(m)} +  \sigma^2_{\varepsilon}},
\end{equation*}
with:
\begin{align*}
 \mu^m_{\mis|\obs} & = \mu^m_{\mis(m)} + \Sigma^m_{\mis(m),\obs(m)} ({\Sigma^m_{\obs(m),\obs(m)}})^{-1}(X_{\obs(m)} - \mu^m_{\obs(m)}), \\
    \Sigma^m_{\mis|\obs} & = \Sigma^m_{\mis(m),\mis(m)} -\Sigma^m_{\mis(m),\obs(m)}( {\Sigma^m_{\obs(m),\obs(m)}})^{-1}\Sigma^m_{\obs(m),\mis(m)}.
\end{align*}
\end{proof}

\subsection{Discussion on how mean-based approaches fail}
\label{app:glm_discussion}

Under \Cref{mod:glm}, the Bayes predictor for a quadratic loss in presence of missing values -- $\mathds{E}\left[ Y | \left( X_{\obs(M)}, M \right) \right]$ -- is fully characterized \citep{lemorvan2020,lemorvan2020neumiss,ayme2022}. 
\Cref{fig:oracles_3d} is obtained by generating the data according to \Cref{mod:glm} with $d = 3$, $\beta = (1, 2, -1)^T$ and $\sigma_\varepsilon = 1$, with multivariate Gaussian $X$ and MCAR mechanism ($X \perp\!\!\!\!\perp M$) (which is a particular case of  \Cref{mod:glm} with $\mu^m \equiv \mu$ and $\Sigma^m \equiv \Sigma$).
The left panel represents the method \textit{Oracle mean + SCP} where SCP is applied on the regressor being the Bayes predictor for the mean with absolute residuals as the score function. The first violin plot represents the marginal coverage whereas the other 7  represent conditional coverage with respect to the different possible patterns: conditional on observing all the variables, on observing all the variables except $X_1$, except $X_2$ etc (see \Cref{sec:experiments} for details on the simulation process). 

\begin{figure}[!h]
    \centering
    \includegraphics[width=0.48\textwidth]{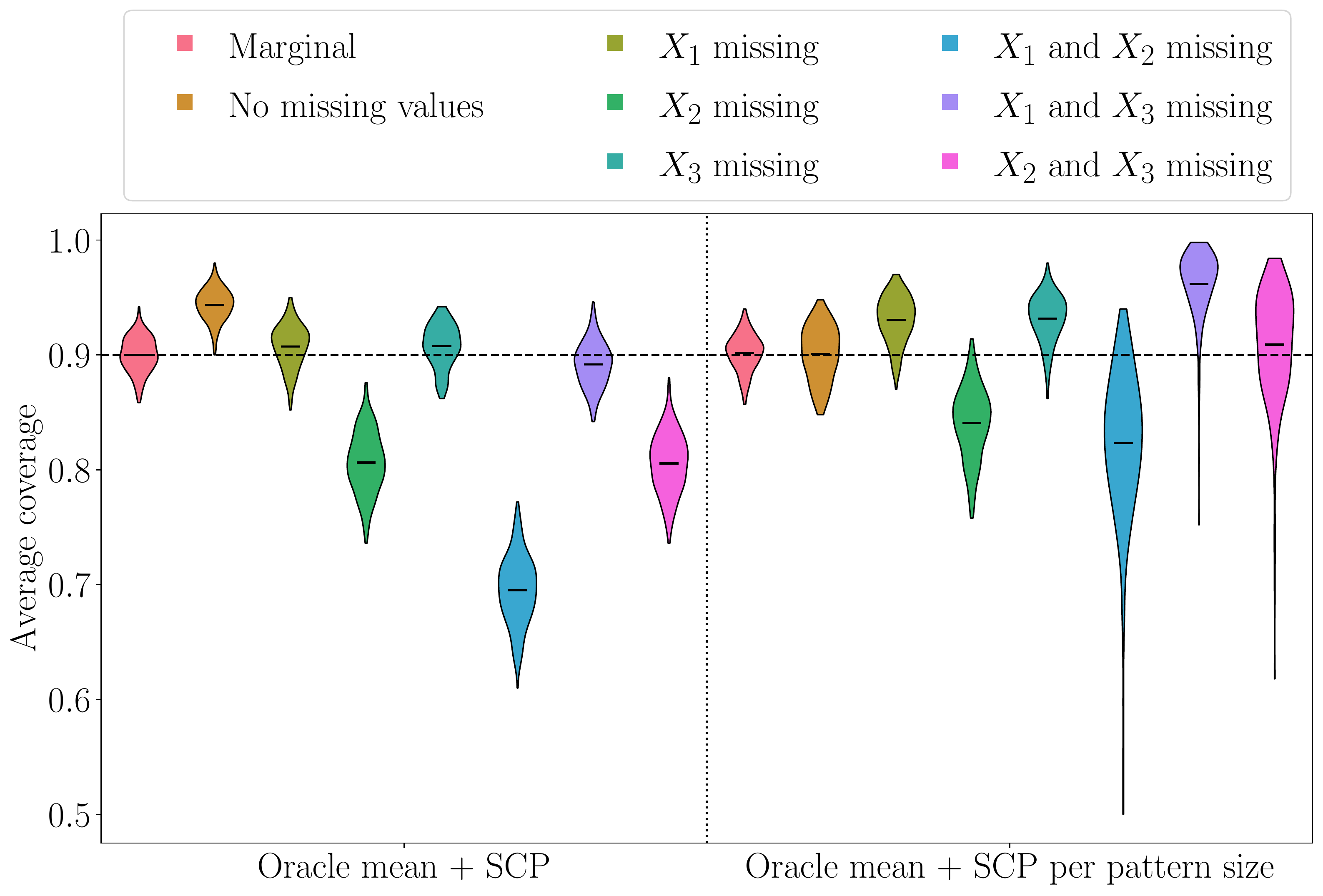}
    \caption{Calibration set contains 500 points. Test size for each pattern is of 500 individuals and for marginal is of 2000. 200 repetitions allow to display violin plots, the horizontal black line representing the mean.}
    \label{fig:oracles_3d}
\end{figure}

\paragraph{SCP on a (oracle) mean regressor lacks of conditional coverage with respect to the mask.} \Cref{fig:oracles_3d} (left) highlights that even with the best mean regressor (the Bayes predictor) and an homoskedastic noise, usual SCP intervals: 
\begin{itemize}[topsep=0pt,noitemsep]
    \item  over-cover when there are no missing values;
    \item  cover less for a mask $\breve{m}$ than for a mask $\mathring{m}$ when $\mathring{m} \subset \breve{m}$ (e.g. $\mathring{m} = (1,0,0)$ only $X_1$ is missing, $\breve{m} = (1,1,0)$ that is $X_1$ and $X_2$ are missing);
    \item cover less when the most informative variable ($X_2$) is missing.
\end{itemize}
To tackle this issue, one could calibrate conditionally to the missing data patterns. This is in the same vein as calibrating conditionally to the categories of a categorical variable or to different groups \citep{romano_malice_2020}. This strategy is not viable as there are $2^d$ patterns: the number of subsets grows exponentially with the dimension, implying the creation of subsets with too little data to perform the~calibration. As an alternative, one could consider to perform calibration conditionally to the pattern size (e.g. when $d=3$, either 0 missing value, 1 or 2). This is possible as there are only $d$ different pattern sizes.  

\paragraph{Calibrating by pattern size does not provide validity conditionally to the missing data patterns.}  \Cref{fig:oracles_3d} (right) shows the coverages of \textit{Oracle mean + SCP per pattern size} where SCP is applied on the Bayes predictor for the mean and the calibration is protected by pattern size. The previous statements still hold with this strategy, even if the coverage disparities are smaller. Therefore, it is not enough to calibrate per pattern size. 

\section{Finite sample algorithms}

\label{app:methods}

\subsection{General statement of \Cref{alg:cp_na_sub}}
\label{app:methods_general}
We provide in \Cref{alg:CPMDA-ex} a general statement of \masksub handling any learning algorithm (both regression and classification) and any score function.

\begin{algorithm}[!h]
\caption{CP-MDA-Exact}
\label{alg:CPMDA-ex}
\begin{algorithmic}[1] 
\REQUIRE Imputation algorithm $\mathcal{I}$, predictive algorithm $\mathcal{A}$, conformity score function $s_g$ paramatrized by a model $g$, significance level $\alpha$, training set $\left\{ \left(X^{(k)},M^{(k)},Y^{(k)}\right)\right\}_{k=1}^n$, test point $\left( X^{(\text{test})},M^{(\text{test})} \right)$.
\ENSURE Prediction interval ${\widehat{C}}_{\alpha}\left( x^{(\text{test})},m^{(\text{test})} \right)$. 
\STATE Randomly split $\{1, \ldots, n\}$ into two disjoint sets $\rm{Tr}$ and $\rm{Cal}$. 
\STATE Fit the imputation function: $\Phi(\cdot) \leftarrow \mathcal{I}\left(\left\{ \left( X^{(k)}, M^{(k)} \right), k \in \rm{Tr}\right\}\right)$ 
\STATE Impute the training set: $\left\{ X_{\text{imp}}^{(k)} \right\}_{k \in \rm{Tr}} := \left\{ \Phi\left(X^{(k)},M^{(k)}\right) \right\}_{k \in \rm{Tr}}$
\STATE Fit algorithm $\mathcal{A}$: $\hat{g}(\cdot) \leftarrow \mathcal{A}\left(\left\{\left( X_{\text{imp}}^{(k)}, Y^{(k)} \right), k \in \rm{Tr}\right\}\right)$ 
\STATEx \textcolor{blindblue}{// Generate an augmented calibration set:}  
\STATE $\rm{Cal^{(\text{test})}} = \left\{ k \in \rm{Cal} \text{ such that } M^{(k)} \subset M^{(\text{test})} \right\}$
\FOR {$k \in \rm{Cal^{(\text{test})}}$}
\STATE $\widetilde M^{(k)} = M^{(\text{test})}$ \COMMENT {Additional masking}

\ENDFOR 
\STATEx \textcolor{blindblue}{Augmented calibration set generated. //} 
\STATE Impute the calibration set: $\left\{ X_{\text{imp}}^{(k)} \right\}_{k \in \rm{Cal^{(\text{test})}}} := \left\{ \Phi\left(X^{(k)},\widetilde M^{(k)}\right) \right\}_{k \in \rm{Cal^{(\text{test})}}}$
\FOR {$k \in \rm{Cal^{(\text{test})}} $} 
\STATE Set $S^{(k)} = s_{\hat g} \left( Y^{(k)}, X_{\text{imp}}^{(k)} \right)$, the \textit{conformity scores}
\ENDFOR
\STATE Set $\mathcal{S}_{\rm{Cal}} = \{S^{(k)}, k \in \rm{Cal^{(\text{test})}}\}$
\STATE Compute $\widehat Q_{1-\tilde\alpha}\left(\mathcal{S}_{\rm{Cal}}\right)$, the $1-\tilde\alpha$-th empirical quantile of $\mathcal{S}_{\rm{Cal}}$, with $1-\tilde\alpha := (1-\alpha)\left(1+1/\#\mathcal{S}_\Cal\right)$.
\STATE Set  ${\widehat{C}}_{\alpha} \left( X^{(\text{test})}, M^{(\text{test})} \right) = \left\{y \text{ such that } s_{\hat g} \left( y, \Phi \left(X^{(\text{test})},M^{(\text{test})} \right) \right) \leq \widehat{Q}_{1-\hat\alpha}\left(\mathcal{S}_{\rm{Cal}}\right)  \right\}$.
\end{algorithmic}
\end{algorithm}

\subsection{Mask-conditional valitidy of \masksub}

\label{app:methods_exact}

Before proving the results, we introduce a slightly stronger notion of mask-conditional-validity, when the calibration set is itself of random cardinality.

\begin{definition}[Mask-conditional-validity-random-calibration-size]
\label{def:cond_cov_rv}
A method is mask-conditionally-valid with a random calibration size $\#\rm{Cal}$ if for any $m \in \mathcal{M}$, the lower bound is satisfied, and exactly mask-conditionally-valid if for any $m \in \mathcal{M}$, $1\le c\le n$, the upper bound is also satisfied:
\begin{align*}
1-\alpha \!\!\! \underset{\text{valid}}{\leq} \!\! \mathds{P}\left(  Y ^{(n+1)} \in \widehat{C}_{\alpha} \left( X^{(n+1)},m \right) | M^{(n+1)} = m,  \#\rm{Cal} = c \right)\underset{\text{exactly valid}}{\leq} 1 - \alpha + \frac{1}{c + 1}.
\end{align*}

\end{definition}

We start by proving \Cref{prop:meth_cond_2} that implies the result on \masksub in \Cref{prop:meth_cond}.

\begin{theorem}\label{prop:meth_cond_2}[Conditional validity of \masksub with calibration of random cardinality]
Assume the missing mechanism is MCAR, and that \Cref{ass:iid,ass:y_ind_m,ass:imp_sym} hold. Then:
\begin{itemize}[topsep=0pt,noitemsep,leftmargin=*,wide]
	\item \masksub is valid with a random calibration size $\#\rm{Cal}$ conditionally to the missing patterns;
	\item if the scores $S^{(k)}$ are almost surely distinct, \masksub is exactly mask-conditionally-valid with a random calibration size $\#\rm{Cal}$.
\end{itemize}
\end{theorem}

\begin{proof}[\underline{Proof of  \Cref{prop:meth_cond_2}}]
Let $\Tr$ and $\Cal$ be two disjoint sets on $\llbracket 1,n \rrbracket$. Let $\hat g$ be some model.
Given \ref{ass:iid}, the sequence $\left\{ \left( X^{(k)}, M^{(k)}, Y^{(k)} \right)_{k \in \rm{Cal}}, \left( X^{(\test)}, M^{(\test)}, Y^{(\test)} \right) \right\}$ is exchangeable. 
Therefore, the sequence $\left\{ \left( X^{(k)}, Y^{(k)} \right)_{k \in \rm{Cal}}, \left( X^{(\test)}, Y^{(\test)} \right) \right\}$ is also exchangeable.

Let $m$ in $\mathcal{M}$.  
We define $\Cal^m = \left\{ k \in \rm{Cal} \text{ such that } M^{(k)} \subset m \right\}$.

Let $c \in \llbracket 1,\#\Cal \rrbracket$.

As the $M \ind X$ (missingness is MCAR) and $(M \ind Y)\vert X$ (\Cref{ass:y_ind_m}), then $M \ind (X,Y)$, and $\#\Cal^m \ind \left( X^{(k)}, Y^{(k)} \right)_{k \in \rm{Cal}}, \left( X^{(\test)}, Y^{(\test)} \right)$. It follows that the sequence $\left\{ \left( X^{(k)}, Y^{(k)} \right)_{k \in \textcolor{blindblue}{\Cal^m}}, \left( X^{(\test)}, Y^{(\test)} \right) \right\}$ is exchangeable conditionally to $\#\Cal^m=c$.

Similarly,  $M^{(\test)} \ind \left( X^{(k)}, Y^{(k)} \right)_{k \in \rm{Cal}}, \left( X^{(\test)}, Y^{(\test)} \right)$. Thus
the sequence $\lbrace \left( X^{(k)}, M^{(\test)}, Y^{(k)} \right)_{k \in \textcolor{blindblue}{\Cal^m}}, \left( X^{(\test)}, M^{(\test)}, Y^{(\test)} \right) \rbrace$ is exchangeable conditionally to $\#\Cal^m=c$ and $M^{(\test)}=m$. 

Therefore, we can now invoke \Cref{prop:marg_cp_na} in combination with Lemma 1 of \citet{romano_malice_2020} to conclude the proof. But we can state a more rigorous version here, since in fact $\Cal^m$ is a random variable (as discussed in \Cref{def:cond_cov_rv}).

Since the algorithm $\mathcal{I}$ treats the calibration and test data points symmetrically (\ref{ass:imp_sym_T} with $\mathcal T = \rm{Cal} \bigcup \rm{Test}$), \ref{ass:imp_sym_T} also holds for any $\mathcal{T}' \subset \mathcal{T}$. Therefore, by \Cref{lem:exch_imp_T} the sequence $\left\{ \left( \Phi(X^{(k)}, M^{(\test)}), M^{(\test)}, Y^{(k)} \right)_{k \in \Cal^m}, \left( \Phi(X^{(\test)}, M^{(\test)}), M^{(\test)}, Y^{(\test)} \right) \right\}$ is exchangeable conditionally to $\#\Cal^m=c$ and $M^{(\test)}=m$.

The conclusion follows from usual arguments \citep{papadopoulos_inductive_2002,lei_distribution-free_2018,angelopoulos-gentle}.

Precisely, $\left\{ \left( s_{\hat g}(Y^{(k)} , \Phi(X^{(k)}, M^{(\test)})) \right)_{k \in \Cal^m},  s_{\hat g}(Y^{(\test)} , \Phi(X^{(\test)}, M^{(\test)}))\right\}$ is exchangeable conditionally to $\#\Cal^m = c$ and ${M^{(\test)}=m}$.
Therefore,
$$ \mathds{P}\left(s_{\hat g}(Y^{(\test)} , \Phi(X^{(\test)}, M^{(\test)})) \leq \widehat Q_{1-\tilde\alpha}( ( s_{\hat g}(Y^{(k)} , \Phi(X^{(k)}, M^{(\test)})) )_{k \in \Cal^m} ) \Big \vert M^{(\test)}=m, \#\Cal^m = c \right) \geq 1 - \alpha,$$
and if the $\left( \left( s_{\hat g}(Y^{(k)} , \Phi(X^{(k)}, M^{(\test)})) \right)_{k \in \Cal^m},  s_{\hat g}(Y^{(\test)} , \Phi(X^{(\test)}, M^{(\test)}))\right)$ are almost surely distinct (i.e. have a continuous distribution) then \citep{lei_distribution-free_2018,romano_conformalized_2019}:
$\resizebox{\textwidth}{!}{$\mathds{P}\left(s_{\hat g}(Y^{(\test)} , \Phi(X^{(\test)}, M^{(\test)})) \leq \widehat Q_{1-\tilde\alpha}\big( ( s_{\hat g}(Y^{(k)} , \Phi(X^{(k)}, M^{(\test)})) )_{k \in \Cal^m} \big) \Big  \vert M^{(\test)}=m, \#\Cal^m = c \right) \leq 1 - \alpha + \frac{1}{c+1}.$}$

This proves the first two points (with respect to \Cref{def:cond_cov_rv}) of \Cref{prop:meth_cond}, by observing that $\left\{ Y^{(\test)} \in \widehat C_{\alpha}(X^{(\test)}, M^{(\test)}) \right\} = \left\{ s_{\hat g}(Y^{(\test)} , \Phi(X^{(\test)}, M^{(\test)})) \leq \widehat Q_{1-\tilde\alpha}\left( \left( s_{\hat g}(Y^{(k)} , \Phi(X^{(k)}, M^{(\test)})) \right)_{k \in \Cal^m}\right) \right\}$.

\end{proof}

Then, the proof of \Cref{prop:meth_marg} (marginal validity of the \masksub) is direct by marginalizing the result of \Cref{prop:meth_cond}. \qed

\subsection{Validities of \mask.}
\label{app:methods_nested}
Next, we give more details on the results on \mask.

\subsubsection{Mask-conditional-validity of \mask.}

Let $m \in \mathcal{M}$.

 We start by describing the links between \mask and \masksub. \masksub can be re-written in the same way as \mask, but keeping the subselection step of l.~\ref{line_algo_1:compute_Caltest}. 

Indeed, first mention that the output of \Cref{alg:cp_na_sub} can be written in the following ways:
\begin{itemize}[noitemsep,leftmargin=*]
    \item 
$\widehat{C}_{\alpha} ( \textcolor{blindpurple}{ X^{(\test)}, m^{(\test)} } ) =  \left[  \hat{q}_{\frac{\alpha}{2}} \circ \Phi ( X^{(\test)}, m^{(\test)} ) - \widehat Q_{1-\tilde\alpha}\left(S\right) ;  \hat{q}_{1-\frac{\alpha}{2}} \circ \Phi ( X^{(\test)}, m^{(\test)} ) + \widehat Q_{1-\tilde\alpha}\left(S\right) \right]$
\item $\widehat{C}_{\alpha} ( \textcolor{blindpurple}{ X^{(\test)}, m^{(\test)} } ) =  \left[  \widehat Q_{\tilde\alpha}\left( \hat{q}_{\frac{\alpha}{2}} \circ \Phi ( X^{(\test)}, m^{(\test)} ) - S_{\Cal^{(\test)}}\right) ;   \widehat Q_{1-\tilde\alpha}\left(\hat{q}_{1-\frac{\alpha}{2}} \circ \Phi ( X^{(\test)}, m^{(\test)} ) + S_{\Cal^{(\test)}} \right) \right]$
\item $\widehat{C}_{\alpha} ( \textcolor{blindpurple}{ X^{(\test)}, m^{(\test)} } ) =  \left[  \widehat Q_{\tilde\alpha}\left( Z^{m^{(\test)}}_{\frac{\alpha}{2}} \right) ;   \widehat Q_{1-\tilde\alpha}\left( Z^{m^{(\test)}}_{1-\frac{\alpha}{2}} \right) \right]$.
\end{itemize}

With $Z^m_{\frac{\alpha}{2}} := \{z^{(k)}_{\frac{\alpha}{2}}, k \in \rm{Cal} \text{ and } \widetilde M^{(k)} = m \}$, and similarly for the upper bag. 
Recall that we have:
$    z^{(k)}_{\frac{\alpha}{2}} = \hat{q}_{\frac{\alpha}{2}} \circ \Phi\left( \textcolor{blindgreen}{ x^{(\text{test})},\widetilde m^{(k)} }\right) - s^{(k)}.
$

On the other hand, the output predictive interval of \Cref{alg:cp_na_jack} is then written as:
\begin{itemize}[noitemsep,leftmargin=*]
    \item 
$ {\widehat{C}}_{\alpha} \left( \textcolor{blindpurple}{ X^{(\test)}, m^{(\test)} } \right) = [ \widehat Q_{\tilde\alpha}\left(Z_{\frac{\alpha}{2}}\right) ; \widehat Q_{1-\tilde\alpha}\left(Z_{1-\frac{\alpha}{2}}\right) ]$.
\end{itemize}

With these notations, $Z_{\frac{\alpha}{2}}$ can be partitioned as \begin{equation}\label{eq:partition}
Z_{\frac{\alpha}{2}} = Z^m_{\frac{\alpha}{2}}\bigcup \Bigg( \bigcup\limits_{\widetilde m ^{(k)} \supset m} Z^{\widetilde m^{(k)}}_{\frac{\alpha}{2}} \Bigg) .
\end{equation}

With 
\begin{align*}
Z_{\frac{\alpha}{2}} &= \{Z^{(k)}_{\frac{\alpha}{2}}, k \in \rm{Cal}\}\\
Z^{(k)}_{\frac{\alpha}{2}} &= \hat{q}_{\frac{\alpha}{2}} \circ \Phi\left( \textcolor{blindgreen}{ X^{(\text{test})},\widetilde M^{(k)} }\right) - S^{(k)} \\
s^{(k)} &= 
 \max( \hat{q}_{\frac{\alpha}{2}}(x_{\text{imp}}^{(k)}) - y^{(k)} , y^{(k)} - \hat{q}_{1-\frac{\alpha}{2}}(x_{\text{imp}}^{(k)}) ).
\end{align*}

The result of \Cref{alg:cp_na_sub} implies that for any mask $m\in \mathcal{M}$, we have :
\begin{align*}
    \mathds P\left(Y^{(\test)} \in  {\widehat{C}}_{\alpha} \left( \textcolor{blindpurple}{ X^{(\test)}, m } \right) | M^{(\test) }= m \right) \geq 1-\alpha ,
\end{align*}
i.e.
\begin{align}
    \mathds P\left(Y^{(\test)} \notin  \left[  \hat{q}_{\frac{\alpha}{2}} \circ \Phi ( X^{(\test)}, m ) - \widehat Q_{1-\tilde\alpha}\left(S^{m}\right) ;  \hat{q}_{1-\frac{\alpha}{2}} \circ \Phi ( X^{(\test)}, m ) + \widehat Q_{1-\tilde\alpha}\left(S^m\right) \right] | M^{(\test) }= m \right) &\leq \alpha.  \label{eq:cov_interval_algo1_m} 
\end{align}
Where: $Q_{1-\tilde\alpha}\left(S\right)$ is the $(1-\alpha)(1+1/\#S)$-quantile of $S$ and $S^m= \{s^{(k)} \text{ for } k \in \rm{Cal} \text{ and } \widetilde M^{(k)} = m \}$.
Equivalently:
\begin{align}\label{eq:pattern_exact}
\mathds P\left(Y^{(\test)} \in   \left[  \widehat Q_{\tilde\alpha}\left( Z^m_{\frac{\alpha}{2}} \right) ;   \widehat Q_{1-\tilde\alpha}\left( Z^{m}_{1-\frac{\alpha}{2}} \right) \right] | M^{(\test) }= m \right) &\geq 1-\alpha .
\end{align}

In the following Lemma, we show that for $\tilde m \supset m$ the result extends under \Cref{ass:sto_dom}.

\begin{lemma}\label{lem:aux}
Assume~\Cref{ass:sto_dom}. For any $m\in \mathcal M$, for any $\tilde m \supset m$
\begin{align}\label{eq:pattern_other_conservative}
 \mathds P \left[  \left(Y^{(\test)} \in   \left[  \widehat Q_{\tilde\alpha}\left( Z^{\tilde m}_{\frac{\alpha}{2}} \right) ;   \widehat Q_{1-\tilde\alpha}\left( Z^{\tilde m}_{1-\frac{\alpha}{2}} \right) \right]  \right) | M^{(\test) }= m\right] &\geq 1-\alpha. 
\end{align}
This inequality shows the conservativeness of the quantiles of the bags resulting from larger missing patterns $\tilde m$ than $m$ when the construction of the output of \Cref{alg:cp_na_jack}. 

While inequality \Cref{eq:pattern_exact} is ``tight'' in the sense that the probability is almost exactly $1-\alpha$ (item 2 of \Cref{prop:meth_cond}), the proof hereafter shows that \Cref{eq:pattern_other_conservative} can be pessimistic in terms of actual coverage, as  one may have $\mathds P [  (Y^{(\test)} \textcolor{blindorange}{\notin}   [  \widehat Q_{\tilde\alpha}( Z^{\tilde m}_{\frac{\alpha}{2}} ) ;   \widehat Q_{1-\tilde\alpha}( Z^{\tilde m}_{1-\frac{\alpha}{2}} ) ]  ) | M^{(\test) }= m] \ll \alpha$. 

More precisely, we have the following inequality:
{
\begin{align}\label{eq_lem:main}
  \resizebox{\linewidth}{!}{ $ {\mathds E} \left[ \mathds P\left(Y^{(\test)} \notin  \left[  \hat{q}_{\frac{\alpha}{2}} \circ \Phi ( X^{(\test)}, \tilde m ) - \widehat Q_{1-\tilde\alpha}\left(S^{\tilde m}\right) ;  \hat{q}_{1-\frac{\alpha}{2}} \circ \Phi ( X^{(\test)}, \tilde m ) + \widehat Q_{1-\tilde\alpha}\left(S^{\tilde m}\right) \right] \bigg | M^{(\test) }= m, X^{(\test)}_{\obs(m)} \right)  \bigg | M^{(\test) }= m \right ] \leq \alpha $ }.
\end{align}
}

\end{lemma}
The interpretation of that Lemma is that the intervals resulting from the prediction on $x^{\test}, \tilde m$ (more data hidden) and corrected with the residuals of the calibration points $(X^k, M^k = \tilde m, Y^k)$ have a \textit{larger} probability of containing $Y^{\test}$, conditionally to $X_{\obs(m)}$ than the interval built using  prediction on $x^{\test},  m$ (more  data available) and corrected with the residuals of the calibration points $(X^k, M^k =  m, Y^k)$ (more data available)

\begin{proof}[Proof of~\Cref{lem:aux}]
We start by invoking \Cref{eq:cov_interval_algo1_m} for $\tilde m$:
\begin{align}
    \mathds P\left(Y^{(\test)} \notin  \left[  \hat{q}_{\frac{\alpha}{2}} \circ \Phi ( X^{(\test)}, \tilde m ) - \widehat Q_{1-\tilde\alpha}\left(S^{\tilde m}\right) ;  \hat{q}_{1-\frac{\alpha}{2}} \circ \Phi ( X^{(\test)}, \tilde m ) + \widehat Q_{1-\tilde\alpha}\left(S^{\tilde m}\right) \right] | M^{(\test) } = \tilde m \right) \leq \alpha .
\end{align}
Consequently, by the tower property of conditional expectations:

\begin{align} \label{eq:cond_exp_proba_cov}
  \resizebox{\linewidth}{!}{   $\mathds E\left [\mathds P\left(Y^{(\test)} \notin  \left[  \hat{q}_{\frac{\alpha}{2}} \circ \Phi ( X^{(\test)}, \tilde m ) - \widehat Q_{1-\tilde\alpha}\left(S^{\tilde m}\right) ;  \hat{q}_{1-\frac{\alpha}{2}} \circ \Phi ( X^{(\test)}, \tilde m ) + \widehat Q_{1-\tilde\alpha}\left(S^{\tilde m}\right) \right] \bigg| M^{(\test) } = \tilde m ,  S^{(\tilde m)}, X^{(\test)}_{\obs(\tilde m)} \right) \bigg| M^{(\test) } = \tilde m  \right] \leq \alpha $ }.
\end{align}
Observe that $\hat{q}_{\frac{\alpha}{2}} \circ \Phi ( X^{(\test)}, \tilde m ) - \widehat Q_{1-\tilde\alpha}\left(S^{\tilde m}\right) $ is $\{M^{(\test) }=\tilde m,  S^{(\tilde m)}, X^{(\test)}_{\obs(\tilde m)}\}$-measurable.

Moreover, by \Cref{ass:sto_dom}, we have that for any $\delta \in [0,0.5]$: 
\begin{align}
q_{1-\delta/2}^{Y | (X_{\obs\left(m\right)}, M = m)} \leq q_{1-\delta/2}^{Y | (X_{\obs\left(\widetilde{m}\right)}, M = \widetilde{m})} \label{eq:dom1}\\
q_{\delta/2}^{Y | (X_{\obs\left({m}\right)}, M = {m})} \geq q_{\delta/2}^{Y | (X_{\obs\left(\widetilde{m} \right)}, M = \widetilde{m})} . \label{eq:dom2}
\end{align}
In other words the conditional distribution of $Y$ given $X_{\obs({\widetilde m})}$ and $M = {\widetilde m}$ ``stochastically dominates'' the conditional distribution of $Y$ given $X_{\obs({m})}$ and $M = {m}$.

We thus have, with $F_Z$ denoting the cumulative distribution function of $Z$:
$F_{Y | (X_{\obs\left(\widetilde{m}\right)}, M = \widetilde{m})}$ the cumulative distribution function of $Y | (X_{\obs\left(\widetilde{m}\right)}, M = \widetilde{m})$:
\begin{footnotesize}\begin{align}
    &\mathds P\left(Y^{(\test)} \notin  \left[  \hat{q}_{\frac{\alpha}{2}} \circ \Phi ( X^{(\test)}, \tilde m ) - \widehat Q_{1-\tilde\alpha}\left(S^{\tilde m}\right) ;  \hat{q}_{1-\frac{\alpha}{2}} \circ \Phi ( X^{(\test)}, \tilde m ) + \widehat Q_{1-\tilde\alpha}\left(S^{\tilde m}\right) \right] \bigg| M^{(\test) } = \tilde m ,  S^{(\tilde m)}, X^{(\test)}_{\obs(\tilde m)} \right) \nonumber \\  
    & = 1 - \left[F_{Y | (X_{\obs \left(\widetilde{m}\right)}, M = \widetilde{m})} \left( \hat{q}_{1-\frac{\alpha}{2}} \circ \Phi ( X^{(\test)}, \tilde m ) + \widehat Q_{1-\tilde\alpha}(S^{\tilde m})\right) - F_{Y | (X_{\obs \left(\widetilde{m}\right)}, M = \widetilde{m})} \left(\hat{q}_{\frac{\alpha}{2}} \circ \Phi ( X^{(\test)}, \tilde m ) - \widehat Q_{1-\tilde\alpha}(S^{\tilde m})\right)\right] \nonumber \\
     & \overset{(i)}{\geq} 1 - \left[F_{Y | (X_{\obs \left({m}\right)}, M = {m})} \left( \hat{q}_{1-\frac{\alpha}{2}} \circ \Phi ( X^{(\test)}, \tilde m ) + \widehat Q_{1-\tilde\alpha}(S^{\tilde m})\right) - F_{Y | (X_{\obs \left({m}\right)}, M = {m})} \left(\hat{q}_{\frac{\alpha}{2}} \circ \Phi ( X^{(\test)}, \tilde m ) - \widehat Q_{1-\tilde\alpha}(S^{\tilde m})\right)\right] \nonumber \\
      & = \mathds P\left(Y^{(\test)} \notin  \left[  \hat{q}_{\frac{\alpha}{2}} \circ \Phi ( X^{(\test)}, \tilde m ) - \widehat Q_{1-\tilde\alpha}\left(S^{\tilde m}\right) ;  \hat{q}_{1-\frac{\alpha}{2}} \circ \Phi ( X^{(\test)}, \tilde m ) + \widehat Q_{1-\tilde\alpha}\left(S^{\tilde m}\right) \right] \bigg| M^{(\test) } = m ,  S^{(\tilde m)}, X^{(\test)}_{\obs( m)} \right) . \label{eq:final_dom}
\end{align}
At (i) we use \eqref{eq:dom2} $F_{Y | (X_{\obs \left({m}\right)}, M = {m})} (\hat{q}_{\frac{\alpha}{2}} \circ \Phi ( X^{(\test)}, \tilde m ) - \widehat Q_{1-\tilde\alpha}(S^{\tilde m}))  \le  F_{Y | (X_{\obs \left(\widetilde{m}\right)}, M = \widetilde{m})} (\hat{q}_{\frac{\alpha}{2}} \circ \Phi ( X^{(\test)}, \tilde m ) - \widehat Q_{1-\tilde\alpha}(S^{\tilde m})) $, and  \eqref{eq:dom1}: $F_{Y | (X_{\obs \left({m}\right)}, M = {m})} (\hat{q}_{1-\frac{\alpha}{2}} \circ \Phi ( X^{(\test)}, \tilde m ) + \widehat Q_{1-\tilde\alpha}(S^{\tilde m}))  \geq  F_{Y | (X_{\obs \left(\widetilde{m}\right)}, M = \widetilde{m})} (\hat{q}_{1-\frac{\alpha}{2}} \circ \Phi ( X^{(\test)}, \tilde m ) + \widehat Q_{1-\tilde\alpha}(S^{\tilde m})) $ by \ref{ass:sto_dom}.
\end{footnotesize}
Remark that here we assume that $\left(\hat{q}_{1-\frac{\alpha}{2}} \circ \Phi ( X^{(\test)}, \tilde m ) + \widehat Q_{1-\tilde\alpha}(S^{\tilde m})\right) \geq \rm{median} (Y^{(\test)} | (X^{(\test)}_{\obs \left(\widetilde{m}\right)}, M = {\tilde m})$ and $\left(\hat{q}_{\frac{\alpha}{2}} \circ \Phi ( X^{(\test)}, \tilde m ) - \widehat Q_{1-\tilde\alpha}(S^{\tilde m})\right) \leq \rm{median} (Y^{\test} | (X^{(\test)}_{\obs \left(\widetilde{m}\right)}, M = \widetilde{m})$.

We obtain \Cref{eq_lem:main} in \Cref{lem:aux} by plugging \eqref{eq:final_dom} in \eqref{eq:cond_exp_proba_cov}, then \Cref{eq:pattern_other_conservative} by the tower property.
\end{proof}

\begin{theorem}
    \label{th:precise_conservative}
Assume the missing mechanism is MCAR, and that \Cref{ass:iid,ass:y_ind_m,ass:imp_sym} hold. Additionally \Cref{ass:sto_dom} is satisfied.

 Consider the partition described in \Cref{eq:partition}, and consider \mask running on a test point with missing pattern $m^{(\test)}$, with any of the following outputs, instead of l.~\ref{line_algo_2:compute_int} $ {\widehat{C}}_{\alpha} \left( \textcolor{blindpurple}{ x^{(\text{test})}, m^{(\text{test})} } \right) = [ \widehat Q_{\tilde\alpha}\left(Z_{\frac{\alpha}{2}}\right) ; \widehat Q_{1-\tilde\alpha}\left(Z_{1-\frac{\alpha}{2}}\right) ]$:
 \begin{enumerate}
     \item $ {\widehat{C}}_{\alpha} \left( \textcolor{blindpurple}{ x^{(\text{test})}, m^{(\text{test})} } \right) = [ \widehat Q_{\tilde\alpha}(Z^{\tilde m}_{\frac{\alpha}{2}}) ; \widehat Q_{1-\tilde\alpha}(Z^{\tilde m }_{1-\frac{\alpha}{2}}) ]$ where $\tilde m \supset m^{(\test)}$  is an arbitrary choice.
     \item $ {\widehat{C}}_{\alpha} \left( \textcolor{blindpurple}{ x^{(\text{test})}, m^{(\text{test})} } \right) = [ \widehat Q_{\tilde\alpha}(Z^{\hat m}_{\frac{\alpha}{2}}) ; \widehat Q_{1-\tilde\alpha}(Z^{\hat m }_{1-\frac{\alpha}{2}}) ]$ where $\hat m$ is a randomly selected pattern in $\{\tilde m , \tilde m \supset m^{(\test)}\}$, possibly with varying probability depending on the cardinality of the sets $Z^{\tilde m}_{\alpha/2}$ .
 \end{enumerate}
Then the resulting algorithm is mask-conditionally-valid.
\end{theorem}

\begin{proof}[\underline{Proof of \Cref{th:precise_conservative}}]
The proof immediately follows from \Cref{eq:pattern_other_conservative}, and gives the result without difficulty for any arbitrary pattern or random variable independent of all other randomness.

Extension to a choice that involves the cardinality of the sets  $Z^{\tilde m}_{\alpha/2}$, leveraging the independence between these cardinals and the coverage properties (same as in the proof of \Cref{prop:meth_cond_2}).
\end{proof}

Then, the proof of \Cref{prop:meth_marg} (marginal validity of the CP-MDA-Nested) is direct by marginalizing the result of \Cref{th:precise_conservative}. \qed

\section{Infinite data results}
\label{app:infinite_data}

\begin{customProposition}[\ref{prop:qr_bayes} ($\ell_\beta$-consistency of an universal learner)]

Let $\beta \in [0,1]$. If $X$ admits a density on $\mathcal{X}$, then, for almost all imputation function $\Phi \in \mathcal{F}^{I}_{\infty}$, the function $g^*_{{\ell_\beta}, \Phi} \circ \Phi$ is Bayes optimal for the pinball risk of level $\beta$. 
\end{customProposition}

\begin{proof}[Proof of \Cref{prop:qr_bayes}]

The proof starts in the exact same way than \citet{lemorvan2021}, based on their Lemmas A.1 and A.2. For completeness, we copy here the statements of these lemmas without their proof and rewrite the two first parts of the main proof. 

Let $\Phi$ be an imputation function such that for each missing data pattern $m$, $\phi^m \in \mathcal{C}^\infty\left(\mathds{R}^{|\obs(m)|}, \mathds{R}^{|\mis(m)|}\right)$. 

\begin{lemma}[Lemma A.1 in \citet{lemorvan2021}]
Let $\phi^m \in \mathcal{C}^{\infty}\left(\mathds{R}^{|\obs(m)|}, \mathds{R}^{|\mis(m)|}\right)$ be the imputation function for missing data pattern $m$, and let $\mathcal{M}^m=\left\{x \in \mathds{R}^d: x_{\mis(m)}=\phi^m\left(x_{\obs((m))}\right)\right\}$. For all $m$, $\mathcal{M}^m$ is an $|\obs((m))|$-dimensional manifold.
\label{lem:A1}
\end{lemma}

In \Cref{lem:A1}, $\mathcal{M}^m$ represents the manifold in which the data points are sent once imputed by $\phi^m$. \Cref{lem:A1} states that this manifold is of dimension $|\obs(m)|$.

\begin{lemma}[Lemma A.2 in \citet{lemorvan2021}]
Let $m$ and $m^{\prime}$ be two distinct missing data patterns with the same number of missing (resp. observed) values $|\mis|$ (resp $|\obs|$). Let $\phi^m \in \mathcal{C}^{\infty}\left(\mathds{R}^{|\obs(m)|}, \mathds{R}^{|\mis(m)|}\right)$ be the imputation function for missing data pattern $m$, and let $\mathcal{M}^m=\left\{x \in \mathds{R}^d: x_{\mis(m)}=\phi^m\left(x_{\obs(m)}\right)\right\}$. We define similarly $\Phi^{\left(m^{\prime}\right)}$ and $\mathcal{M}^{\left(m^{\prime}\right)}$. For almost all imputation functions $\phi^m$ and $\Phi^{\left(m^{\prime}\right)}$,
$$
\operatorname{dim}\left(\mathcal{M}^m \cap \mathcal{M}^{\left(m^{\prime}\right)}\right)= \begin{cases}0 & \text { if }|\mis|>\frac{d}{2} \\ d-2|\mis| & \text { otherwise. }\end{cases}
$$
\label{lem:A2}
\end{lemma}

Note that, as by \Cref{lem:A1} $\operatorname{dim}\left( \mathcal{M}^m \right) = \operatorname{dim}\left(\mathcal{M}^{(m')}\right) = |\obs| = d - |\mis|$, \Cref{lem:A2} states that $\operatorname{dim}\left(\mathcal{M}^m \cap \mathcal{M}^{\left(m^{\prime}\right)}\right) \leq \operatorname{dim}\left(\mathcal{M}^m\right) = \operatorname{dim}\left(\mathcal{M}^{(m')}\right)$.

Now, to prove \Cref{prop:qr_bayes} the missing data patterns are ordered as in \citet{lemorvan2021}: the first one will be the one in which all the variables are missing, while the last one will be the one in which all the variables are observed. For two data patterns with the same number of missing variables, the ordering is picked at random. We denote by $m(i)$ the $i$-th missing data pattern according to this ordering. 

We are going to build a function $g_{\Phi}$ which, composed with $\Phi$, will reach the $\ell$-Bayes risk.

For each missing data pattern, and starting by $m(1)$ of all variables missing, we can define $g_\Phi$ on the data points from the current missing data pattern. More precisely, for each $i$, $g_\Phi$ is built for every imputed data point belonging to $\mathcal{M}^{(m(i))}$ except for those already considered in previous steps (one imputed data point can belong to multiple manifolds):

\begin{equation*}
\forall Z=\Phi(X,M) \in \mathcal{M}^{(m(i))} \backslash \bigcup_{k<i} \mathcal{M}^{(m(k))}, \quad g^{\star}(Z)=\tilde{f}^{\star}(\widetilde{X})
\end{equation*}

That is, $g_\Phi \circ \Phi (X,M) $ will equal $\tilde{f}^* (X,M)$ except possibly if $\Phi(X,M) = \Phi(\tilde{Y})$ for some $\tilde{Y}$ that has more missing values than $X,M$. Therefore, for each missing data pattern $m(i)$, $g_\Phi \circ \Phi$ equals $\tilde{f}^*$ except on $\bigcup_{k<i} \mathcal{M}^{(m(k))}$. The question that remains is: what is the dimension of $\mathcal{M}^{(m(i))} \bigcap \left( \bigcup_{k<i} \mathcal{M}^{(m(k))} \right)$, these points for which there is no necessarily equality between $g_\Phi \circ \Phi$ and $\tilde{f}^*$. First, note that $\mathcal{M}^{(m(i))} \bigcap \left( \bigcup_{k<i} \mathcal{M}^{(m(k))} \right) = \bigcup_{k<i}  \left(\mathcal{M}^{(m(i))} \bigcap \mathcal{M}^{(m(k))} \right)$. For each space in this reunion, there are two cases:
\begin{itemize}
    \item either $|\obs(m(k))| < |\obs(m(i))|$: using \Cref{lem:A1}, $\operatorname{dim}\left(\mathcal{M}^{(m(k))}\right) =  |\obs(m(k))| < |\obs(m(i))| = \operatorname{dim}\left(\mathcal{M}^{(m(i))}\right)$. Thus, $\mathcal{M}^{(m(i))} \bigcap \mathcal{M}^{(m(k))}$ is of measure zero in $\mathcal{M}^{(m(i))}$.
    \item either $|\obs(m(k))| = |\obs(m(i))|$: using \Cref{lem:A2}, $\mathcal{M}^{(m(i))} \bigcap \mathcal{M}^{(m(k))}$ is of dimension 0 or smaller than $\operatorname{dim}\left(\mathcal{M}^{(m(i))} \right)$, thus it is of measure zero in $\mathcal{M}^{(m(i))}$.
\end{itemize}

Therefore, the set of data points for which $g_\Phi \circ \Phi$ does not equal the oracle is of measure 0 for each missing data pattern.

Let $\beta \in [0,1]$. We can now write down the $\ell_\beta$-risk of this built function:

\begin{align*}
\mathds{E}\left[ \ell_\beta \left( Y,g^* \circ \Phi(X,M) \right) \right] & = \mathds{E}\left[ \rho_\beta \left( Y - g^* \circ \Phi(X,M) \right) \right] \\
& = \mathds{E}\left[ \rho_\beta \left( Y - \tilde{f^*}(X,M) +  \tilde{f^*}(X,M) - g^* \circ \Phi (X,M) \right) \right] \\
(i) & \leq \mathds{E}\left[ \rho_\beta \left( Y - \tilde{f^*}(X,M) \right) \right] +  \mathds{E}\left[ \rho_\beta \left( \tilde{f^*}(X,M) - g^* \circ \Phi(X,M) \right) \right] \\
 & \leq \mathcal{R}_{\ell_\beta}^* +  \mathds{E}\left[ \rho_\beta \left( \tilde{f^*}(X,M) - g^* \circ \Phi(X,M) \right) \right],
\end{align*}

where $(i)$ holds thanks to the shape of $\rho_\beta$. For any $w \in \mathds{R}$ and any $\lambda \in \mathds{R}_+$:
\begin{align*}
    \rho_\beta\left(\lambda w\right) & = \beta \lambda |w| \mathds{1}_{w \geq 0} + (1-\beta) \lambda |w| \mathds{1}_{w \leq 0} \\
     \rho_\beta\left(\lambda w\right) & = \lambda \rho_\beta\left(w\right).
\end{align*}
Furthermore, $\rho_\beta$ is convex, thus for any $(u,v) \in \mathds{R}^2$:
\begin{align*}
    \rho_\beta\left(\frac{1}{2}u+\frac{1}{2}v\right) & \leq \frac{1}{2}\rho_\beta(u)+\frac{1}{2}\rho_\beta(v) \\
    \frac{1}{2} \rho_\beta\left(u+v\right) & \leq \frac{1}{2}\rho_\beta(u)+\frac{1}{2}\rho_\beta(v) \\
    \rho_\beta\left(u+v\right) & \leq \rho_\beta(u)+\rho_\beta(v).
\end{align*}

As $\tilde{f^*}$ and $g^* \circ \Phi$ are equals almost everywhere on each missing subspace,  $\mathds{E}\left[ \rho_\beta \left( \tilde{f^*}(X,M) - g^* \circ \Phi (X,M) \right) \right] = 0$. Indeed, decomposing by pattern one can write:
\begin{align*}
\mathds{E}\left[ \rho_\beta \left( \tilde{f^*}(X,M) - g^* \circ \Phi (X,M) \right) \right] & = \sum_{M = m} \mathds{P}(M = m) \mathds{E}\left[ \rho_\beta \left( \tilde{f^*}(X,M) - g^* \circ \Phi (X,M) \right) | M=m \right]
\end{align*}
and thus by equality almost everywhere for each pattern every term in this sum is null.

Therefore one obtains:
\begin{equation*}
\mathds{E}\left[ \ell_\beta \left( Y, g^* \circ \Phi(X,M) \right) \right] \leq \mathcal{R}_{\ell_\beta}^*.
\end{equation*}
Thus:
\begin{equation*}
\mathds{E}\left[ \ell_\beta \left( Y, g^* \circ \Phi(X,M) \right) \right] = \mathcal{R}_{\ell_\beta}^*,
\end{equation*}
and $g^* \circ \Phi$ is Bayes optimal. This implies that $\mathcal{R}_{\ell_\beta,\Phi}^* = \mathcal{R}_{\ell_\beta}^*$. Thus, a universally consistent algorithm learning $g_\Phi$ chained with $\Phi$ will lead to a Bayes consistent function.
\end{proof}

\begin{proof}[Proof of \Cref{cor:asymp_cond}]
\Cref{cor:asymp_cond} states that ``For any missing mechanism, for almost all imputation function $\Phi \in \mathcal{F}^{I}_{\infty}$, if $F_{Y|(X_{\obs(M)},M)}$ is continuous, a universally consistent quantile regressor trained on the imputed data set yields asymptotic conditional coverage.''.

Let $\beta \in [0,1]$.

Remark that \Cref{prop:qr_bayes} states that for any missing mechanism, for almost all imputation function $\Phi \in \mathcal{F}^{I}_{\infty}$ a universally consistent quantile regressor trained on the imputed data set achieves the Bayes risk asymptotically. We will thus show that any $\ell_\beta$-Bayes predictor $f_\beta^*$ (any function achieving the $\ell_\beta$-Bayes-risk) is such that $\mathds{P}( Y \leq f_\beta^*(X,M) | X_{\obs(M)}, M) = \beta$ if $F_{Y|(X_{\obs(M)},M)}$ is continuous. Therefore, any two Bayes predictors $f_{\alpha/2}^*$ and $f_{1-\alpha/2}^*$ form an interval $[f_{\alpha/2}^*(X,M); f_{1-\alpha/2}^*(X,M)]$ that achieves conditional coverage (conditionally to $X_{\obs(M)}$ and $M$).

Let $f_\beta^*$ be a $\ell_\beta$-Bayes predictor. Then:
\begin{align*}
f_\beta^* \in \argmin_{f : \mathcal{X} \times \mathcal{M} \rightarrow \mathds{R}} & \mathds{E}\left[ \rho_\beta \left(Y - f\left(X,M\right)\right) \right] \\
 = & \mathds{E}\left[\mathds{E}\left[ \rho_\beta \left(Y - f\left(X,M\right)\right) | X_{\obs(M)},M \right]\right].
\end{align*}

Let $(x,m) \in \mathcal{X}\times\mathcal{M}$. Denote $H_{x,m}(z) := \mathds{E}\left[ \rho_\beta \left(Y - z \right) | X_{\obs(M)} = x_{\obs(m)}, M=m \right]$. 
As $Y \neq z$ almost surely, we have: 
\begin{align*}
H'_{x,m}(z) & = \mathds{E}\left[ - \rho'_\beta \left(Y - z \right) | X_{\obs(M)} = x_{\obs(m)}, M=m \right] \\
& = \mathds{E}\left[ - ( - \beta \mathds{1}_{Y-z \geq 0} + (1-\beta) \mathds{1}_{Y-z \leq 0} ) | X_{\obs(M)} = x_{\obs(m)}, M=m \right] \\
& = \mathds{E}\left[ \beta \mathds{1}_{Y \geq z} - (1-\beta) \mathds{1}_{Y \leq z} | X_{\obs(M)} = x_{\obs(m)}, M=m \right] \\
& =  \beta \mathds{P}\left( Y \geq z | X_{\obs(M)} = x_{\obs(m)}, M=m \right) - (1-\beta) \mathds{P} \left(Y \leq z | X_{\obs(M)} = x_{\obs(m)}, M=m \right) \\
& = \beta \left( 1-\mathds{P}\left( Y \leq z | X_{\obs(M)} = x_{\obs(m)}, M=m \right) \right) - (1-\beta) \mathds{P} \left(Y \leq z | X_{\obs(M)} = x_{\obs(m)}, M=m \right) \\
H'_{x,m}(z) & =  \beta - \mathds{P}\left( Y \leq z | X_{\obs(M)} = x_{\obs(m)}, M=m \right).
\end{align*}

Therefore $H'_{x,m}(z) \leq 0$ if and only if $\beta \leq \mathds{P}\left( Y \leq z | X_{\obs(M)} = x_{\obs(m)}, M=m \right)$.

Thus, $z$ minimizes $H_{x,m}$ if and only if $\beta = \mathds{P}\left( Y \leq z | X_{\obs(M)} = x_{\obs(m)}, M=m \right)$. 

If $F_{Y|(X_{\obs(M)},M)}$ is continuous, there exists at least a solution, that might not be unique if it is not additionally strictly increasing. Therefore, if $F_{Y|(X_{\obs(M)},M)}$ is continuous, all the $\ell_\beta$-Bayes predictors can be written as $f_\beta^*(x,m) = q_{x,m}$ with $\mathds{P}\left( Y \leq q_{x,m} | X_{\obs(M)} = x_{\obs(m)}, M=m \right) = \mathds{P}\left( Y \leq f_\beta^*(x,m) | X_{\obs(M)} = x_{\obs(m)}, M=m \right) = \beta $.

\end{proof}

\section{Experimental study}
\label{app:exp}

\subsection{Settings detail}
\label{app:exp_set}

\paragraph{Quantile Neural Network.} The architecture and optimization of the Quantile Neural Network used in the experiments is taken from \citet{chr} (their code is freely available). This is the description provided in the original paper of the neural network: ``The network is composed of three fully connected layers with a hidden dimension of 64, and ReLU activation functions. We use the pinball loss to estimate the conditional quantiles, with a dropout regularization of rate 0.1. The network is optimized using Adam \citet{adam} with a learning rate equal to 0.0005. We tune the optimal number of epochs by cross validation, minimizing the loss function on the hold-out data points; the maximal number of epochs is set to 2000.''

\subsection{Gaussian linear results}
\label{app:exp_synth}

\begin{figure}[!h]
    \centering
    \includegraphics[width=0.85\textwidth]{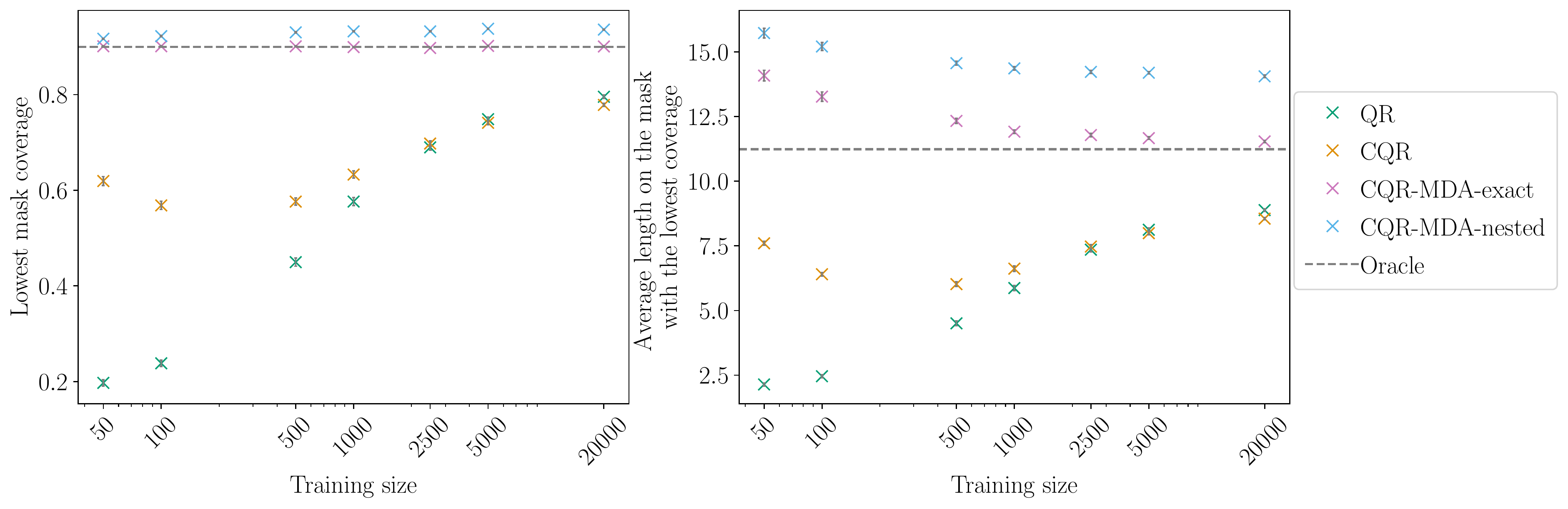}
    \caption{Coverage and interval's length for the mask leading to the lowest coverage. Model is NN. Calibration size fixed to 1000. The mask is concatenated in the features. Data is imputed using Iterative Ridge. 100 repetitions allow to display error bars, corresponding to standard error.}
    \label{fig:poc_masking_worst}
\end{figure}

\Cref{fig:poc_masking_best} is the analogous of \Cref{fig:poc_masking_worst}, but by evaluating the performances on the \camerareadyrevisionlast{mask} leading to the highest coverage.

\begin{figure}[!h]
    \centering
    \includegraphics[width=0.85\textwidth]{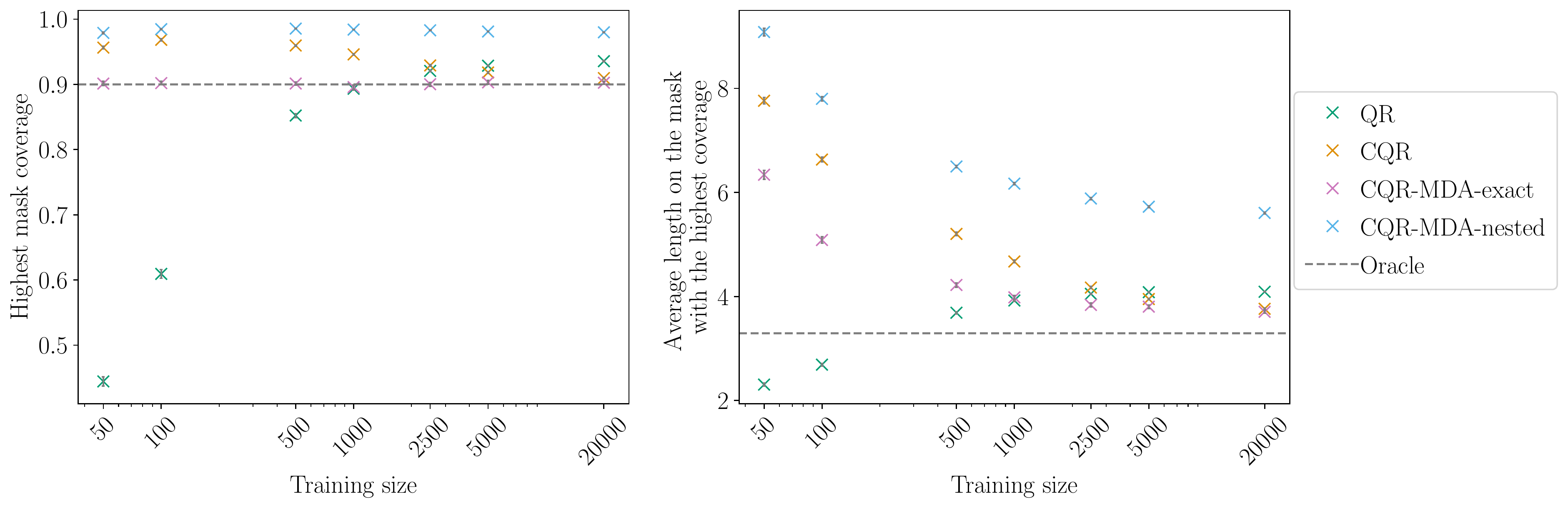}
    \caption{Coverage and interval's length for the mask leading to the highest coverage. See caption of \Cref{fig:poc_masking_worst} for the setting.}
    \label{fig:poc_masking_best}
\end{figure}

Hereafter, we present in \Cref{fig:patternsize} the exact same figure than \Cref{fig:d10_cov_len} but with a panel (the first) for vanilla QR. The 3 other methods are displayed again to facilitate the comparison.

\begin{figure}[!h]
    \centering
    \includegraphics[width=0.95\textwidth]{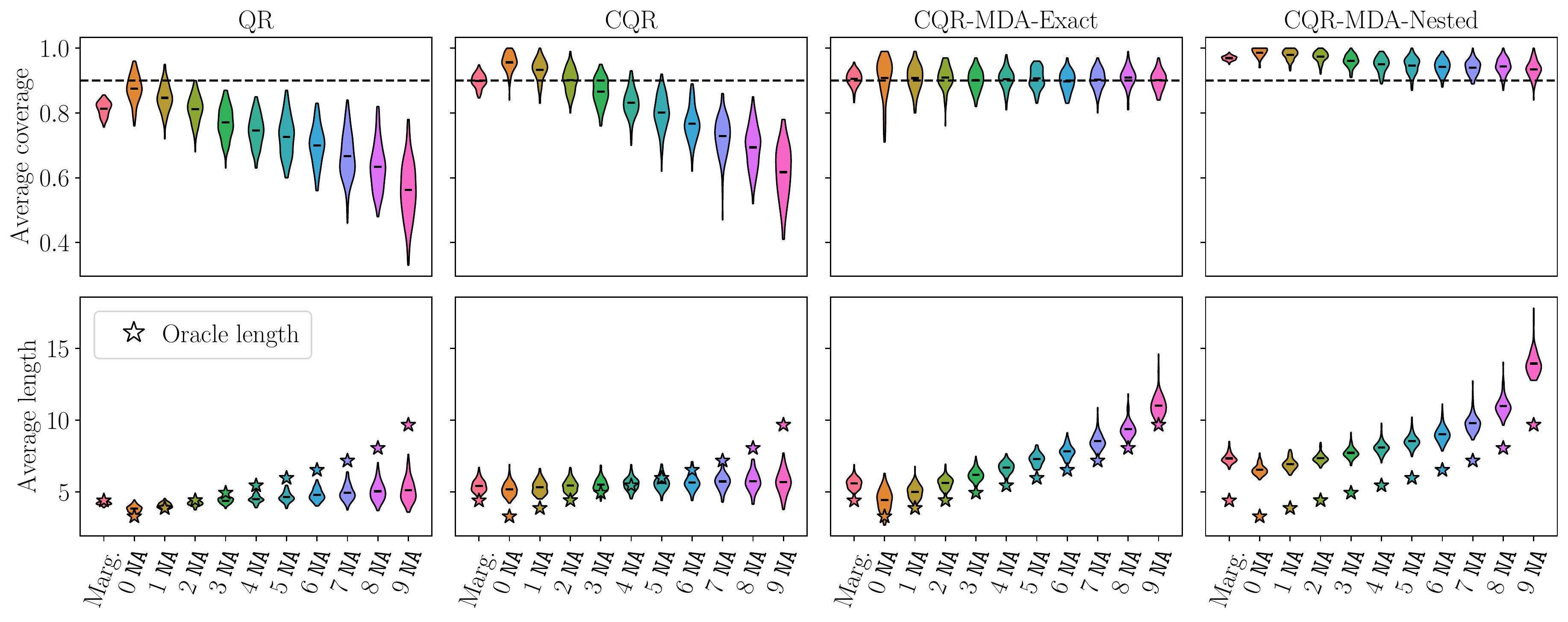}
    \vspace{-10pt}
    \caption{Average coverage (top) and length (bottom) as a function of the pattern size, i.e. the number of missing values (\texttt{NA}). First violin plot corresponds to marginal coverage. Stars correspond to the oracle length. Settings are: model is NN, train size is 500, calibration size is 250. The marginal test set includes 2000 observations. The conditional test set includes 100 individuals for each possible missing data pattern size. The mask is concatenated to the features. Data is imputed using Iterative Ridge. 100 repetitions are performed.}
    \label{fig:patternsize}
\end{figure}

Finally, \Cref{fig:d10_cov_len_1000} is the analogous of \Cref{fig:patternsize}, but for a training set containing 1000 observations and a calibration set containing 500 observations.

\begin{figure}[!h]
    \centering
    \includegraphics[width=0.95\textwidth]{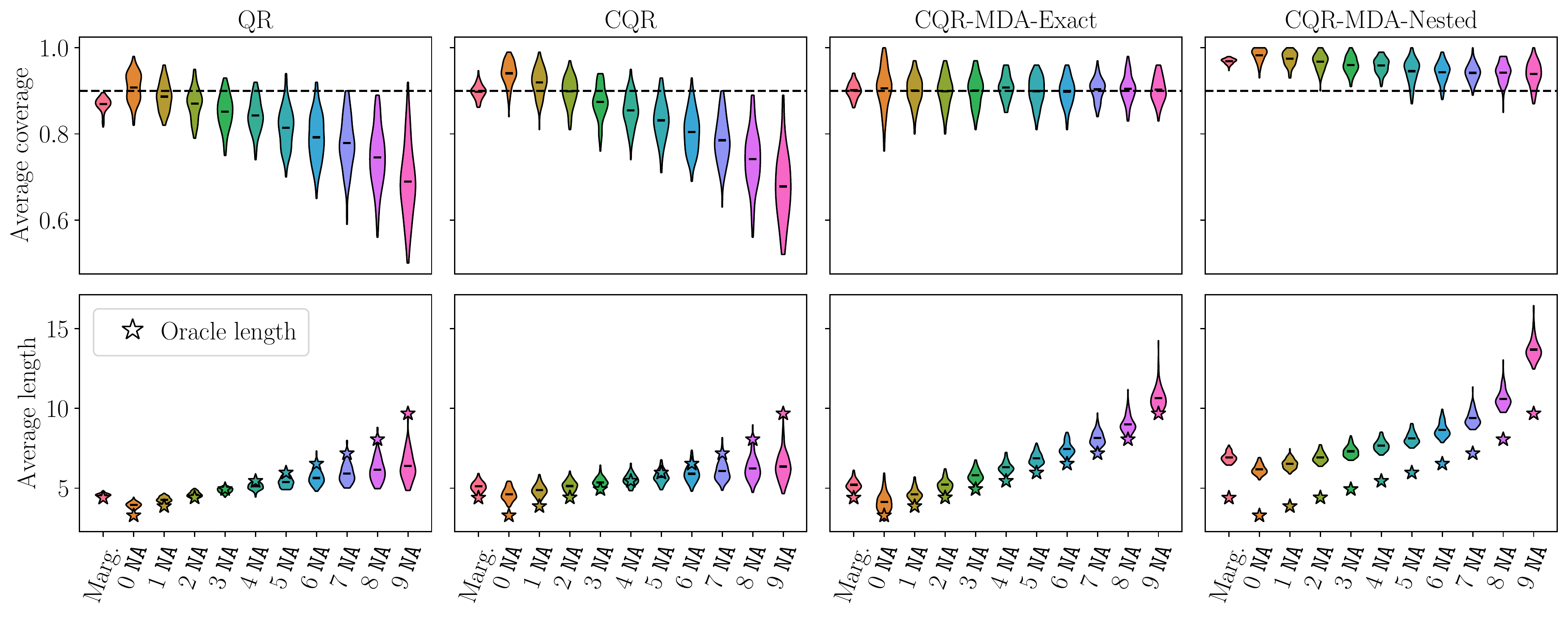}
    \vspace{-10pt}
    \caption{Model is NN. Train size is 1000. Calibration size fixed to 500. The marginal test set includes 2000 observations. The conditional test set includes 100 individuals for each possible missing data pattern size. The mask is concatenated in the features. Data is imputed using Iterative Ridge. 100 repetitions are performed.}
    \label{fig:d10_cov_len_1000}
\end{figure}

\camerareadyrevisionlast{
\subsection{Higher proportion of missing values}
\label{app:exp_more_na}

We present synthetic experiments where the proportion of MCAR missing values is of 40\% (instead of 20\% in \Cref{fig:d10_cov_len}). Except from this, the settings are exactly the same than the ones of \Cref{fig:d10_cov_len}. Precisely, the data is generated with $d=10$ according to \Cref{mod:glm}, with $X \sim \mathcal{N}\left(\mu, \Sigma \right)$, $\mu = (1,\cdots,1)^T$ and $\Sigma = \varphi (1,\cdots,1)^T(1,\cdots,1)+(1-\varphi)I_d$, $\varphi=0.8$, Gaussian noise $\varepsilon \sim \mathcal{N}(0,1)$ and the following regression coefficients $\beta = (1, 2, -1, 3, -0.5, -1, 0.3, 1.7, 0.4, -0.3)^T$. For each pattern size, 100 observations are drawn according to the distribution of $M | \text{size}(M)$ in the test set. The training and calibration sizes are respectively 500 and 250. The experiment is repeated 100 times. The results are displayed in \Cref{fig:d10_cov_len_more_na}.

\begin{figure}[!h]
    \centering
    \includegraphics[width=0.95\textwidth]{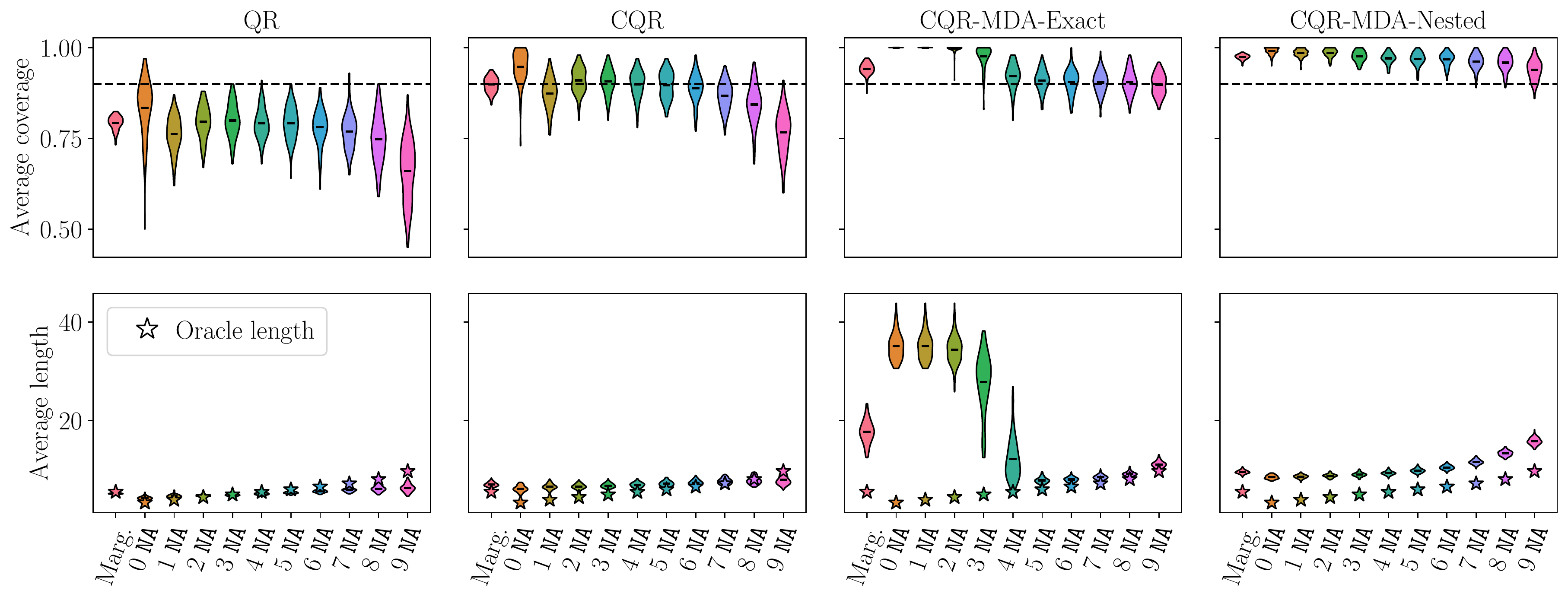}
    \vspace{-10pt}
    \caption{Same caption than \Cref{fig:patternsize}.}
    \label{fig:d10_cov_len_more_na}
\end{figure}

Interestingly, although expected, these experiments lead \masksub to frequently output infinite intervals. This is because the subsampling step with few calibration data -- with respect to the dimension and proportion of missing values -- reached a point where there are not enough observations for \masksub to calibrate accurately for some patterns. 

To compare \masksub and \mask in this setting, \Cref{fig:d10_cov_len_more_na} is obtained by replacing the infinite intervals by $\underset{k \in \mathrm{Tr} \cup \mathrm{Cal}}{\text{max}} y^{(k)} - \underset{k \in \rm{Tr} \cup \rm{Cal}}{\text{min}} y^{(k)}$. It highlights that \masksub is less \emph{efficient} (i.e. outputs larger intervals) than \mask for patterns with less than 4 \texttt{NA}s. With a smaller calibration set or a higher proportion of missing values, this effect would be amplified and generalized to more patterns. \Cref{fig:d10_cov_len_more_na} also emphasizes that \masksub leads to more coverage variability than \mask, on the patterns for which \masksub does not almost surely cover. 

}

\subsection{Semi-synthetic experiments}
\label{app:exp_semi_synth}

In the semi-synthetic experiments, two settings are examined: one where the training size is small in comparison to the number of parameters of the Neural Network -- ``Medium'' --, and one where the training size is even smaller so that some masks have a really low (or null) frequency of appearance in the training set -- ``Small''. In both cases, the calibration size is approximately half the training size. \Cref{fig:semi_synth} presented the results for the ``Medium'' case. 
\vspace{-5pt}
\begin{table}[!h]
\caption{Semi-synthetic settings: training and calibration sizes for each of the 6 data sets depending on the setting.}
\label{tab:semi_synth}
\begin{center}
\begin{small}
\begin{tabular}{|c|c|c|c|c|c|c|c|}
\toprule
\multicolumn{2}{|c|}{} & \texttt{meps_19} & \texttt{meps_20} & \texttt{meps_21} & \texttt{bio} & \texttt{bike} & \texttt{concrete} \\
\multicolumn{2}{|c|}{} & $d = 139$, $l = 5$ & $d = 139$, $l = 5$ & $d = 139$, $l = 5$ & $d = 9$, $l = 9$ & $d = 18$, $l = 4$ & $d = 8$, $l = 8$ \\
\multicolumn{2}{|c|}{} & $n = 15785$ & $n = 17541$ & $n = 15656$ & $n = 45730$ & $n = 10886$ & $n = 1030$ \\
\midrule
\multirow{2}{*}{Small} & $\rm{Tr}$ size  & 500 & 500 & 500 & 500 & 500 & 330 \\
& $\rm{Cal}$ size & 250 & 250 & 250 & 250 & 250 & 100 \\
\midrule
\multirow{2}{*}{Medium}  & $\rm{Tr}$ size & 1000 & 1000 & 1000 & 1000 & 1000 & 630 \\
& $\rm{Cal}$ size & 500 & 500 & 500 & 500 & 500 & 200 \\

\bottomrule
\end{tabular}
\end{small}
\end{center}
\end{table}

\Cref{fig:semi_synth_both} represents the results for these settings, using the same parameters than \Cref{fig:semi_synth}. For the results on the two other \texttt{meps} data sets (\texttt{meps_20} and \texttt{meps_21}) see \Cref{fig:semi_synth_meps}, which repeats the visualisation of \texttt{meps_19} to ease comparison.

\begin{figure}[!h]
    \centering
    \includegraphics[width=0.96\textwidth]{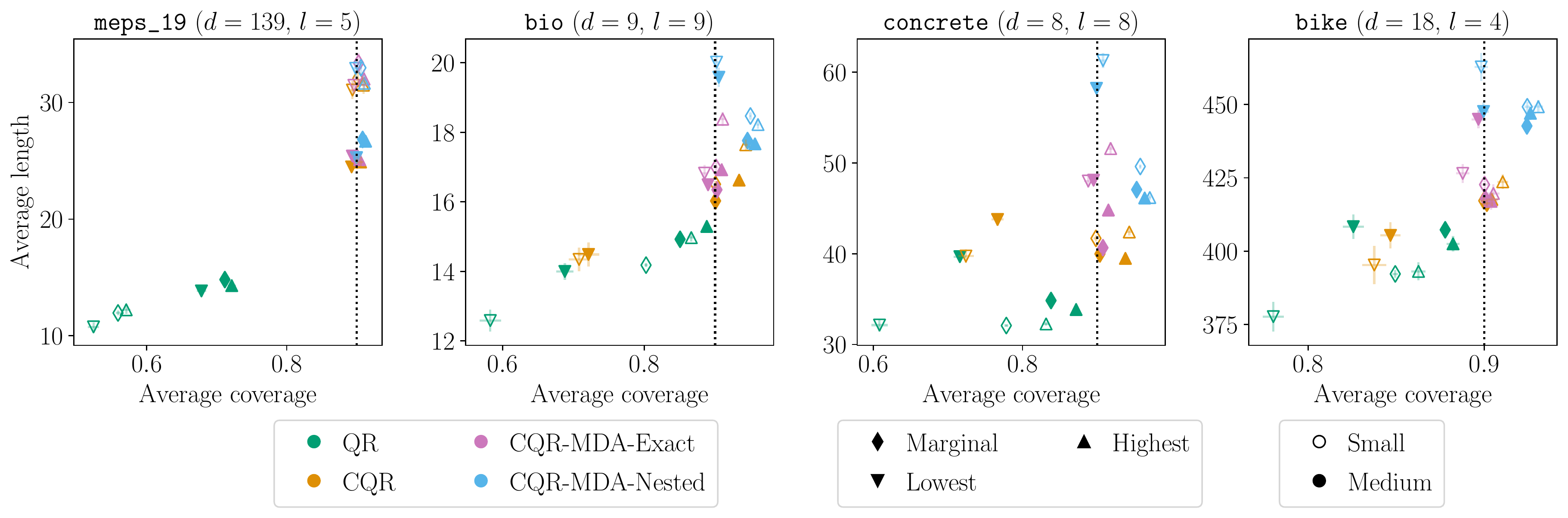}
    \caption{Model is NN. The mask is concatenated in the features. Data is imputed using Iterative Ridge. 100 repetitions are performed, allowing to display the standard error as error bars. The vertical dotted lines represent the target coverage of 90\%.}
    \label{fig:semi_synth_both}
\end{figure}

\begin{figure}[!h]
    \centering
    \includegraphics[width=0.88\textwidth]{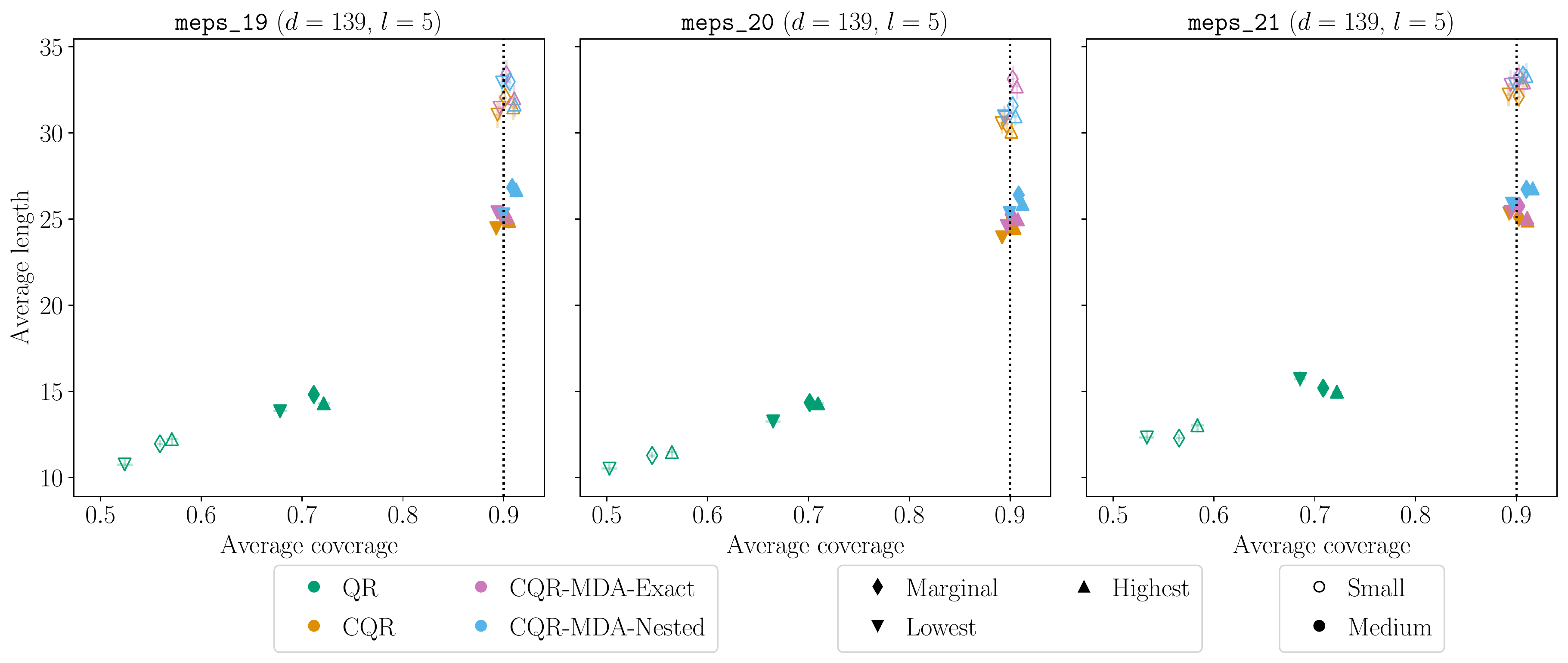}
    \caption{Same caption than \Cref{fig:semi_synth_both}.}
    \label{fig:semi_synth_meps}
\end{figure}

\subsection{Real data}
\label{app:exp_real}

\textbf{Data set description.} \citet{sportisse} selected 7 variables to model the level of platelets, after discussion with medical doctors. Thus, we followed their pipeline. Here are the 7 variables used:
\begin{itemize}
    \item \texttt{Age}: the age of the patient (no missing values);
    \item \texttt{Lactate}: the conjugate base of lactic acid, upon arrival at the hospital (17.66\% missing values);
    \item \texttt{Delta_hemo}: the difference between the hemoglobin upon arrival at hospital and the one in the ambulance (23.82\% missing values);
    \item \texttt{VE}: binary variable indicating if a Volume Expander was applied in the ambulance. A volume expander is a type of intravenous therapy that has the function of providing volume for the circulatory system (2.46\% missing values);
    \item \texttt{RBC}: a binary index which indicates whether the transfusion of Red Blood Cells Concentrates is performed (0.37\% missing values);
    \item \texttt{SI}: the shock index. It indicates the level of occult shock based on heart rate (HR) and systolic blood pressure (SBP), that is SI = $\frac{\text{HR}}{\text{SBP}}$, upon arrival at hospital (2.09\% missing values);
    \item \texttt{HR}: the heart rate measured upon arrival of hospital (1.62\% missing values).
\end{itemize}

\textbf{Splitting strategy.} To study the coverage conditionally on the masks, we must handle the scarcity of some of them. For each individual in the data set, we make only one prediction, this way avoiding too many repetitions of the same test point when computing the average. We split the data set into 5 folds, and predict on each fold by training the procedure on the 4 others, with 15390 observations for training, and 7694 for calibration.

\begin{figure}[!h]
\begin{center}
\includegraphics[width=0.4\textwidth]{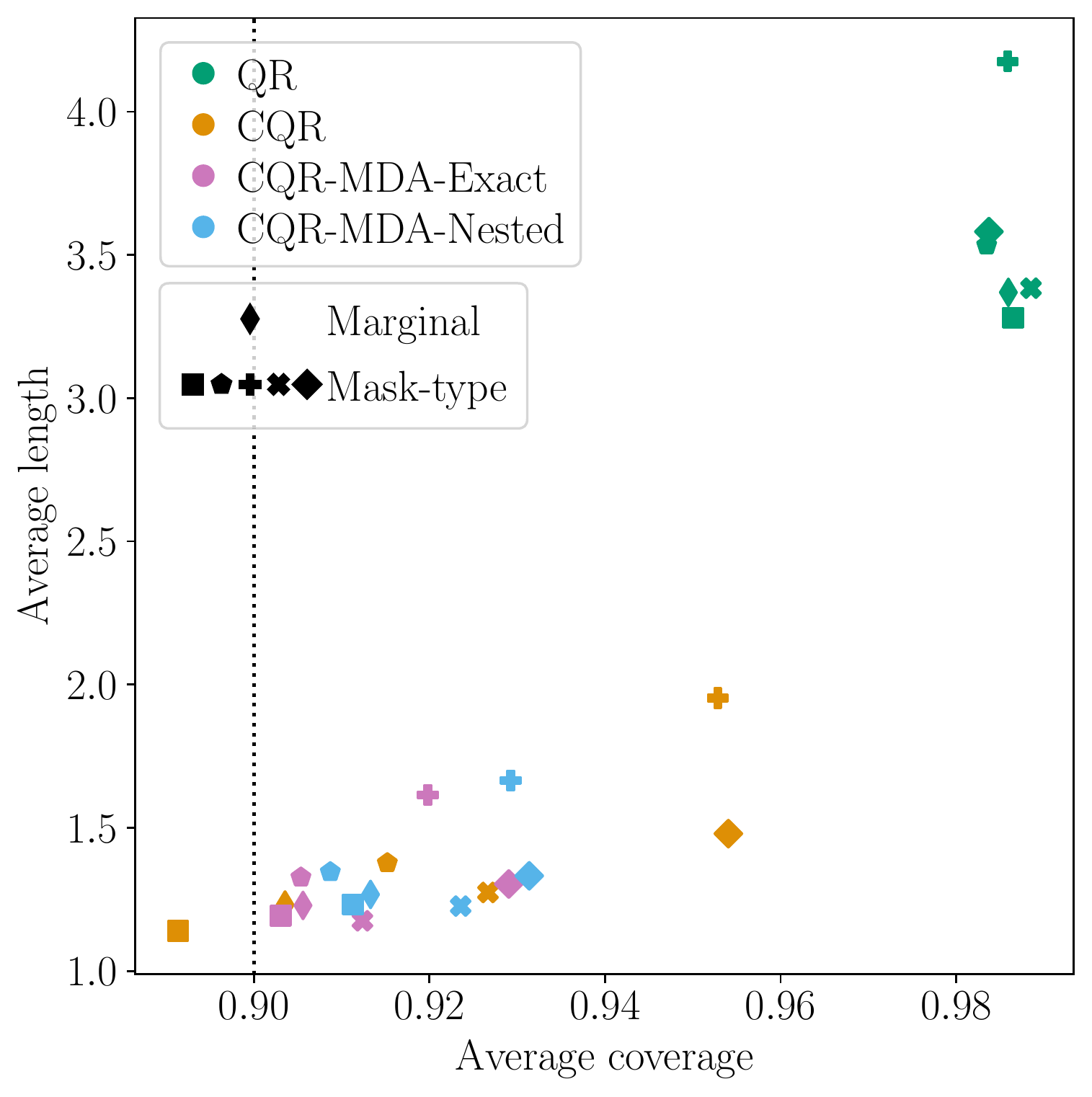}
\end{center}
\caption{Average coverage and length on the TraumaBase® data when predicting the platelets level. Colors correspond to the methods. Diamond ($\blacklozenge$) corresponds to taking the average among all individuals. Other symbols correspond to computing the average among the individuals having a fixed mask. The vertical dotted line represents the target coverage of 90\%. Model is NN. The mask is concatenated to the features. Imputation is Iterative Ridge. Each individual is predicted using 15390 observations for training, and 7694 for calibration.}
\end{figure}

\end{document}